\documentclass{article}

\usepackage[bbgreekl]{mathbbol}
\usepackage{mathrsfs}
\usepackage{graphicx}
\usepackage{amsmath}
\usepackage{amsthm}
\usepackage{amsfonts}
\usepackage{indentfirst}
\usepackage{amssymb}
\usepackage{algorithm, algorithmic}
\usepackage{xcolor}
\usepackage{geometry}
\usepackage{appendix}
\usepackage{bm}
\usepackage{url}
\usepackage{setspace}
\usepackage{subfigure}
\usepackage{verbatim}
\usepackage[normalem]{ulem}
\usepackage{multirow}
\usepackage{lineno}
\usepackage{float}
\usepackage{array}
\usepackage{diagbox}
\usepackage{stmaryrd}
\usepackage{enumerate}
\usepackage{cleveref}
\usepackage[affil-it]{authblk}
\numberwithin{equation}{section}

\newtheorem{assumption}{Assumption}
\newtheorem{definition}{Definition}
\newtheorem{lemma}{Lemma}
\newtheorem{theorem}{Theorem}

\newcommand{\parameter}{\theta}
\newcommand\bmx{\bm{x}}

\newcommand{\lb}{\left[}
\newcommand{\rb}{\right]}

\newcommand{\T}{\mathrm{T}}

\newcommand{\tr}{\operatorname{tr}}

\newcommand{\lrp}[1]{\left(#1\right)}
\newcommand{\lrb}[1]{\left[#1\right]}

\newcommand{\norm}[1]{\left\| #1\right\|}

\newcommand{\EE}{\mathbb{E}}

\newcommand{\iprod}[2]{\left\langle #1, #2 \right\rangle}
\newcommand{\nrm}[1]{\left\|#1\right\|}
\newcommand{\abs}[1]{\left|#1\right|}
\newcommand{\minop}[1]{\min\left\{#1\right\}}

\newcommand{\epct}[1]{\mathbb{E}\left[#1\right]}
\newcommand{\cond}[2]{\mathbb{E}\left[\left.#1\right|#2\right]}

\newcommand{\bigO}[1]{\mathcal{O}\left(#1\right)}

\newcommand{\ceil}[1]{\left\lceil #1\right\rceil}

\newcommand{\pos}[1]{\left[#1\right]_{+}}

\newcommand{\oP}{\operatorname{P}}
\DeclareRobustCommand{\rchi}{{\mathpalette\irchi\relax}}
\newcommand{\irchi}[2]{\raisebox{\depth}{$#1\chi$}}
\newcommand{\Lthre}{\mathcal{D}}
\newcommand{\tparameter}{\tilde{\parameter}}

\newcommand{\ex}[2]{\mathbb{E}_{#1}\left[#2\right]}

\newcommand{\Lcal}{\mathcal{L}}
\newcommand{\Xcal}{\mathcal{X}}
\newcommand{\Acal}{\mathcal{A}}
\newcommand{\Hcal}{\mathcal{H}}
\newcommand{\Ecal}{\mathcal{E}}
\newcommand{\Rcal}{\mathcal{R}}
\newcommand{\Bcal}{\mathcal{B}}
\newcommand{\Pcal}{\mathcal{P}}
\newcommand{\Ebb}{\mathbb{E}}
\newcommand{\Rbb}{\mathbb{R}}

\newcommand{\Pbb}{\mathbb{P}}
\newcommand{\Cbb}{\mathbb{C}}
\newcommand{\Nbb}{\mathbb{N}}
\newcommand{\Zbb}{\mathbb{Z}}

\author{Tianyou Li\thanks{School of Mathematical Sciences, Peking University, CHINA (tianyouli@stu.pku.edu.cn).}
,\quad Fan Chen\thanks{School of Mathematical Sciences, Peking University, CHINA (chern@pku.edu.cn).}
,\quad  Huajie Chen\thanks{School of Mathematical Sciences, Beijing Normal University, CHINA  (chen.huajie@bnu.edu.cn).}
,\quad  Zaiwen Wen\thanks{Beijing International Center for Mathematical Research, Peking University, CHINA (wenzw@pku.edu.cn).}
}

\title{Convergence Analysis of Stochastic Gradient Descent with MCMC Estimators}

\begin{document}

\maketitle
\begin{abstract}
    Understanding stochastic gradient descent (SGD) and its variants is essential for machine learning.  However, most of the preceding
analyses are conducted under amenable conditions such as unbiased gradient estimator and bounded objective functions, which does not encompass many sophisticated applications, such as variational Monte Carlo,  entropy-regularized reinforcement learning and variational inference. In this paper, we consider the SGD algorithm that employ the Markov Chain Monte Carlo (MCMC) estimator to compute the gradient, called MCMC-SGD. Since MCMC reduces the sampling complexity significantly, it is an asymptotically convergent biased estimator in practice. Moreover, by incorporating a general class of unbounded functions, it is much more difficult to analyze the MCMC sampling error. Therefore, we assume that the function is sub-exponential and use the Bernstein inequality for non-stationary Markov chains to derive error bounds of the MCMC estimator.  Consequently, MCMC-SGD is proven to have a first order convergence rate $O(\log K/\sqrt{n K})$ with $K$ iterations and a sample size $n$.  It partially explains how MCMC influences the behavior of SGD.  Furthermore, we verify the correlated negative curvature condition under reasonable assumptions. It is shown that MCMC-SGD escapes from saddle points and reaches $(\epsilon,\epsilon^{1/4})$ approximate second order stationary points or $\epsilon^{1/2}$-variance points at least $O(\epsilon^{-11/2}\log^{2}(1/\epsilon) )$ steps with high probability. Our analysis unveils the convergence pattern of MCMC-SGD across a broad class of stochastic optimization problems, and interprets the convergence phenomena observed in practical applications.
\end{abstract}

\section{Introduction}
\label{sec:introduction}
\setcounter{equation}{0}

We consider a general class of stochastic optimization problems. Given a measurable space $\mathcal{X}$, let $\pi_{\parameter}(x)$ be a family of probability distributions on $\mathcal{X}$ and $f_{\parameter}(x)$ be a function for $x\in \mathcal{X}$. A general form of stochastic optimization problems is stated as
\begin{equation}
    \label{eq:loss}
    \begin{aligned}
        \min_{\parameter\in\Theta}~ \Lcal(\parameter)=\ex{x\sim\pi_{\parameter}}{f_{\parameter}(x)},
    \end{aligned}
\end{equation}
where $\Theta$ denotes the space of parameters, which is typically a subset of a Euclidean space $\mathbb{R}^{d}$. In the field of stochastic optimization, the most common form is typically expressed by a given distribution $\pi$ and a parameterized function $f_{\theta}$. Also, there are numerous tasks, such as reinforcement learning (RL), that aim to find a suitable distribution $\pi_{\theta}$ under a given function $f$. Compared to the two aforementioned problems, \eqref{eq:loss} represents a broader framework, wherein both the distribution $\pi_{\parameter}$ and the function $f_{\theta}$ are dependent on parameters $\theta$. 

A main issue in solving the stochastic optimization problem \eqref{eq:loss} is that the difficulty of directly computing integrals over high-dimensional spaces. The Monte Carlo method offers a strategy for estimating expectations using samples, thereby facilitating the computation of stochastic gradients. However, for certain complex variational distributions, it is prohibitive to obtain their exact probability density functions which impedes directly sampling. 
In order to circumvent the intractable normalization constants in probabilities, Markov chain Monte Carlo (MCMC) is employed to sample from the unnormalized probability. In contrast to the ordinary unbiased Monte Carlo methods, the MCMC method requires some time for mixing and produces the desired samples biasedly and non-independently. There are several common MCMC algorithms, such as Metropolis-Hastings algorithm \cite{metropolis1953equation,hastings1970monte}, Hamiltonian Monte Carlo \cite{duane1987hybrid}, Stochastic Gradient Langevin Dynamics \cite{welling2011bayesian},  etc. The efficiency of the optimization algorithm relies on the error of the MCMC sampling, which has not been adequately investigated in the previous literatures. Previous studies \cite{miasojedow2014hoeffding,jiang2018bernstein,fan2021hoeffding,paulin2015concentration} on the concentration inequality of MCMC sampling typically assume a more idealized setting where the random variables are uniformly bounded. 

With the stochastic gradient estimated by MCMC algorithms, the objective function is minimized using stochastic gradient descent (SGD) methods. The vanilla SGD method is to sample independently from a uniform distribution, and has been extensively studied. Moulines \& Bach first show linear convergence of SGD non-asymptotically for strongly convex functions \cite{moulines2011non}. Needell et al. improve these results by removing the quadratic dependency on the condition number in the iteration complexity results \cite{needell2014stochastic}. Among these convergence results, the gradient noise assumptions for i.i.d samples are of vital importance to establish an $O(1/\sqrt{nK})$ convergence rate for general non-convex cost functions, where $n$ is the sample size per iteration and $K$ is the total number of iterations. However, stochastic sampling is not always independent or unbiased. The convergence of SGD with Markovian gradients has been studied in \cite{duchi2012ergodic, sun2018markov}, and SGD with biased gradient estimators is considered in \cite{ajalloeian2020convergence}. Moreover, there are also some research about SGD with MCMC noise, which we call MCMC-SGD. Atchade et al. consider stochastic proximal gradient algorithms with MCMC sampling on convex objective \cite{atchade2017perturbed}. The uniform-in-$\theta$ ergodic assumption is proposed to analyze the Markov noise for convex bounded functions in their stochastic proximal gradient algorithm. Karimi et al. study SGD with Markovian noise \cite{karimi2019non}. They analyze the convergence by controlling the noise through the Poisson equation. But strong assumptions are given and the convergence result has an extra term that depends on the bias of the estimation scheme.

The contributions of this paper are multifaceted and useful in the advancement of stochastic optimization, particularly in the context of machine learning and scientific computing. Three main contributions are delineated as follows in a structured manner:

\begin{enumerate}
    \item \textbf{Error analysis of MCMC estimators}: A novel aspect of this paper is the application of concentration inequalities to estimate the upper bounds of bias and variance associated with MCMC sampling. Compared with the conventional boundedness assumption in \cite{miasojedow2014hoeffding,jiang2018bernstein,fan2021hoeffding,paulin2015concentration}, we adopt a broader assumption of unbounded functions, recognizing the more complex and realistic scenarios often encountered in MCMC sampling. Based on concentration inequalities, our approach investigates the MCMC error for a class of sub-exponential functions from the perspective of spectral gap. This analysis is universal and non-asymptotic, entailing a meticulous examination of the non-stationary Markov chain on unbounded functions.
    \item \textbf{First-order convergence of MCMC-SGD}: We demonstrate that the biased stochastic gradient, a result of the MCMC sampling, achieves first-order stationary convergence. In comparison with \cite{atchade2017perturbed}, we conduct a detailed analysis on the MCMC error and extend its assumption of convex bounded functions. Meanwhile,  our analysis does not need assumptions with the Poisson equation in \cite{karimi2019non}.  %we circumvent the less intuitive assumptions with the Poisson equation in \cite{karimi2019non}. 
    The convergence rate is quantified as $O\left(\frac{\log K}{\sqrt{n K}}\right)$, which is established after $K$ iterations, given a sufficiently large sample size $N$. This result is instrumental in validating the effectiveness of the MCMC-SGD algorithm and provides a theoretical viewpoint that can be used to guide the design and tuning of the algorithm.
    \item \textbf{Escaping from saddle points and second-order convergence}: Our investigation extends into the realm of second-order convergence properties. The SGD escaping from saddle points with unbiased gradient estimators has been studied in \cite{ge2015escaping,daneshmand2018escaping,jin2017escape}. Under the influence of biased MCMC noise, it is imperative to examine the second-order convergence properties of MCMC-SGD for the problem \eqref{eq:loss}. By substantiating the correlated negative curvature (CNC) condition under errors, we generalize the analysis of unbiased SGD escaping from saddle points in \cite{daneshmand2018escaping} to MCMC-SGD, where the gradient estimator is biased and requires more carefully analysis. A second-order convergence guarantee characterized by a rate of $O\left(\epsilon^{-11/2}\log^{2} \left(\frac{1}{\epsilon}\right)\right)$. This guarantee not only provides theoretical assurance of the convergence speed but also serves specifically as a testament to the algorithm's ability to efficiently approximate eigenvalues for variational eigenvalue problems in quantum mechanics.
\end{enumerate}

The rest of this paper is organized as follows. In Section \ref{sec:model}, we describe the applications of the general optimization problem and introduce the SGD algorithm with MCMC sampling. In Section \ref{sec:conv}, we first give our assumptions for the function and the variational distribution. Then, the sampling error is analyzed asymptotically by the concentration inequality for Markov chains. We prove that the MCMC-SGD algorithm converges to stationary points and provide the convergence rates. In Section \ref{sec:saddle}, the correlated negative curvature condition of the MCMC-SGD algorithm is verified in empirical sampling under reasonable assumptions and the second-order convergence characteristics are further investigated. We establish the convergence guarantee to avoid saddle points with high probability, based on the stochastic perturbations of MCMC algorithms.

\section{SGD with MCMC estimators}

\label{sec:model}
For the problem \eqref{eq:loss}, we notice that parameter dependency exists in both the function and distribution. It is regarded as a more generalized form than purely optimizing distributions or functions in stochastic optimization problems. 

To begin with, we formally present the assumptions regarding $f_{\theta}$ and $\pi_{\theta}$ in problem \eqref{eq:loss}. 
    \begin{assumption}
        \label{asm:loss}
        For the stochastic optimization problem \eqref{eq:loss}, the following two properties holds correspondingly.
        \begin{enumerate}
            \item[(1)] The function $f_{\parameter}(x)$ satisfies
             \begin{equation}
                \label{eq:restriction}
                \ex{x\sim\pi_{\parameter}}{\nabla_{\parameter}f_{\parameter}(x)}=0.
            \end{equation}
            \item[(2)] There exists an energy function  $\phi_{\parameter}(x)$ such that the distribution $\pi_{\parameter}(x)$ is parameterized as
            \begin{equation}
                \label{eq:energy}
                \pi_{\parameter}(x)=\frac{e^{\phi_{\parameter}(x)}}{\int_{x\in \Xcal}e^{\phi_{\parameter}(x)}dx}\propto e^{\phi_{\parameter}(x)},
            \end{equation} 
            where $\phi_{\parameter}(x)$ is a known parameterization.
        \end{enumerate}
    \end{assumption}

    The function $f_{\parameter}(x)$ need not be assumed bounded on the space $\mathcal{X}$, as numerous examples have demonstrated that such an assumption is often difficult to satisfy. Assumption \ref{asm:loss}(1) constitutes an extension to the case where $f$ is independent of the parameter $\theta$, that is, $\nabla_{\parameter}f_{\parameter}(x)=0$ for all $x\in\Xcal$. We naturally generalize the case where $f$ is parameter-independent. The function $f_{\theta}$ with restriction \eqref{eq:restriction} can not only cover a wider range of stochastic optimization problems but also retains the form of policy gradients, facilitating our optimization efforts.
    
    The distribution $\pi_{\parameter}$ is typically viewed as an artificially designed parametric distribution, encompassing these distributions from the exponential family (e.g., Gaussian, Bernoulli, Poisson distribution), structured probabilistic models such as graphical models (e.g., Bayesian networks, Markov random fields), and a series of deep learning architectures that serve as function approximators. In most scenarios, we may not have direct access to the value of $\pi_\theta$, but instead $\pi_\theta$ is parameterized by the energy function $\phi_\theta$. Assumption \ref{asm:loss}(2) expounds that $\pi_{\theta}$ possesses a unnormalized form with the energy function $\phi_{\parameter}$, which frequently occurs in distributions expressed by neural networks. The relation \eqref{eq:energy} is a weak assumption and often appears in the design of complex distributions.
    Since the energy function $\phi_{\theta}$ is typically chosen to be tractable and differentiable with respect to the parameters $\theta$, the computation of $\int_{x\in \Xcal}e^{\phi_{\parameter}(x)}dx$ is often infeasible for complex models or high-dimensional spaces, as it involves an integral over all possible states of $x\in \Xcal$. However, MCMC techniques, such as Metropolis-Hastings algorithm and Stochastic Gradient Langevin Dynamics and so on, can sample from unnormalized distributions, thereby circumventing the need for complex integral computations.

% Both of them rely on the parameter $\parameter$, which needs for optimization on the expectation of $f_{\parameter}(x)$ over $\pi_{\parameter}(x)$. 

To substantiate the validity of this framework, let us consider the following three problems as exemplary cases, where it can be verified that $f_{\parameter}$ is unbounded and satisfies \eqref{eq:restriction}. 
\setcounter{equation}{0}
\subsection{Applications}
\label{subsec:application}
\subsubsection{Variational Monte Carlo for many-body quantum problems} 
\label{subsubsec:VMC}
	Consider a many-body quantum system with $N$ particles. We denote the $N$-particle configuration of this system by $x:=(x_1,\cdots,x_N) \in \Xcal := \Acal^N$, with $\Acal$ being the one-particle configuration space, which can be continuous or discrete. The wavefunction $\Psi:\Xcal\rightarrow\Cbb~({\rm or}~\Rbb)$ describes the quantum state of the many-body system, and is often required to satisfy some symmetric/anti-symmetric conditions. The Hamiltonian $\Hcal$ is a self-adjoint operator acting on the wavefunction, which determines the dynamics and the ground state of the quantum system. 

    A central issue in quantum physics is to compute the ground state energy and wavefunction of the system, which corresponds to the lowest eigenvalue $E_0$ and its corresponding eigenfunction of 
\begin{eqnarray}
\label{eigen}
\Hcal\Psi_0 = E_0\Psi_0 .
\end{eqnarray}
Since it is challenge to  solve the very high dimensional eigenvalue problem \eqref{eigen} directly, an alternative way is to minimize the following Rayleigh quotient
\begin{align}
\label{min:Rayleigh}
E_0 \approx %\min\limits_{\Psi\neq\bm{0}} \frac{\int_{x\in \Xcal} \Psi^{*}(x)\cdot \big(\Hcal\Psi\big)(x)dx}{\int_{x\in\Xcal}\Psi^*(x)\cdot\Psi(x)dx },
\min_{\parameter\in \Theta}\frac{\int_{x\in \Xcal} \Psi_{\parameter}^{*}(x)\cdot \big(\Hcal\Psi_{\parameter}\big)(x)dx}{\int_{x\in\Xcal}\Psi_{\parameter}^*(x)\cdot\Psi_{\parameter}(x)dx },
\end{align}
where the wavefunction is parameterized by a suitable ansatz $\Psi_{\parameter},~ \parameter\in\Theta$ within a finite parameter space. 
When $\Xcal$ is a discrete configuration space, the above integrals in \eqref{min:Rayleigh} is regarded as summations over all configurations. %However, it is intractable to optimize the wavefunction in the function space. 
%Hence, VMC approximates the wavefunction by a suitable ansatz $\Psi_{\parameter},~ \parameter\in\Theta$ within a finite parameter space. 
%There are plenty of common variational ansatzes in VMC, including Jastrow ansatz, restricted Boltzmann machine (RBM), feed forward neural networks (FFNN) and a variety of other neural networks. 
We can regard \eqref{min:Rayleigh} as a stochastic optimization problem by the following reformulation:
\begin{equation}
    \label{eq:vmc}
    \begin{aligned}
        \min_{\parameter\in \Theta}\int_{x\in\Xcal}\underbrace{\frac{\Psi_{\parameter}^*(x)\cdot\Psi_{\parameter}(x)}{\int_{x\in\Xcal}\Psi_{\parameter}^*(x)\cdot\Psi_{\parameter}(x)dx}}_{\pi_{\parameter}(x)}\cdot\underbrace{\frac{\Hcal\Psi_{\parameter}(x)}{\Psi_{\parameter}(x)}}_{f_{\parameter}(x)}dx
        =\ex{x\sim\pi_{\parameter}}{f_{\theta}(x)}.
    \end{aligned}
\end{equation}
As the Hamiltonian $\Hcal$ is a self-adjoint linear operator, the following holds for $f_{\parameter}(x)$:
\begin{equation*}
    \ex{x\sim\pi_{\parameter}}{\nabla_{\theta}f_{\theta}(x)}=\frac{\int_{x\in \Xcal} \Psi_{\parameter}^{*}(x)\cdot \Hcal\nabla_{\theta}\Psi_{\parameter}(x)-\nabla_{\theta}\Psi_{\parameter}^{*}(x)\cdot \Hcal\Psi_{\parameter}(x)dx}{\int_{x\in\Xcal}\Psi_{\parameter}^*(x)\cdot\Psi_{\parameter}(x)dx}=0.
\end{equation*}
    
For examples, the $N$-electron configuration for a many-electron system in $3$ dimension is $x=(x_1,\cdots,x_N)$ with $x_i=(r_i,\sigma_i)\in \Acal=\Rbb^3\times \Zbb_2$, where $r_i$ represents the spatial coordinate and $\sigma_i$ is the spin coordinate.
    The Hamiltonian of the electron system is given by
    \begin{eqnarray}
    \label{hamiltonian:SE}
    \Hcal = -\frac{1}{2}\sum_{i=1}^N\Delta_{r_i} + \sum_{i=1}^{N} v_\mathrm{ext}(r_i) + \sum_{1\leq i<j\leq N}v_\mathrm{ee}(r_i,r_j),
    \end{eqnarray}
    where $v_\mathrm{ext}:\Rbb^3\rightarrow\Rbb$ is the ionic potential and $v_\mathrm{ee}:\Rbb^3\times \Rbb^3\rightarrow\Rbb$ represents the interaction between electrons
    \begin{equation}
    	\begin{aligned}
    		v_\mathrm{ext}(r)=-\sum_{I}\frac{Z_I}{|r-R_I|},~~ v_\mathrm{ee}(r,r^{\prime})=\frac{1}{|r-r^{\prime}|},\quad r,r^{\prime}\in \Rbb^3,
    	\end{aligned}
    \end{equation}
    with $R_I\in \Rbb^3$ and $Z_I\in \Nbb $ being the position and atomic number of the $I$-th nuclear.

    In quantum many-body systems, the classical VMC employs MCMC algorithms to compute approximate gradients, which garnered renewed interest due to the rise of modern neural network ansatzes. These techniques typically rely on the ability of artificial neural networks to represent complex high-dimensional wavefunction, which have already been explored in many fields of physics and chemistry. The RBM is first proposed by Carleo and Troyer \cite{carleo17} as a variational ansatz for many-body quantum problems. Furthermore, a large number of deep neural networks, such as FFNN \cite{saito17,cai18}, deep RBMs \cite{deng17,glasser18,nomura17,kaubruegger2018}, convolutional neural networks \cite{liang2018,choo2018}, variational autoencoders \cite{rocchetto18}, have been applied to capture the physical features and improve the accuracy of the ground state. Motivated by the traditional Slater-Jastrow-backflow ansatzes, PauliNet \cite{hermann2020deep} and FermiNet \cite{pfau2020ab} use the permutation equivariant and invariant construction of networks and many determinants to approximate a general antisymmetric wavefunction. However, the convergence properties of VMC methods have remained underexplored in a long time. A recent study \cite{abrahamsen2023convergence} by Abrahamsen et al. provides a proof of convergence for VMC methods with an unbiased gradient estimator. 
\subsubsection{Entropy-regularized reinforcement learning}
In RL, exploration in actions is critical to finding good policies during optimization. Without a large number of diverse state-action pairs, the algorithm might settle for a weak policy. To prevent the concentration in policy, the entropy regularization is used to encourage random exploration \cite{williams1991function,mnih2016asynchronous}. Let $x:=(s_0,a_0,s_1,a_1,\dots)$ be the trajectory and $r(x)$ be the cumulative reward over the trajectory $x$. Compared to the classical reinforcement learning, the introduction of entropy regularization for the parameterized policy $\pi_{\parameter}(x)$ yields a distinct objective function, that is, 
	\begin{equation}
        \label{eq:entropy-RL}
        \min_{\parameter\in \Theta}~\ex{x\sim\pi_{\parameter}}{-r(x)}+\beta \ex{x\sim\pi_{\parameter}}{\log\pi_{\parameter}(x)}=\Ebb_{x\sim\pi_{\parameter}}\Big[\underbrace{-r(x)+\beta \log \pi_{\parameter}(x)}_{f_{\parameter}(x)}\Big].\end{equation}
    Since $\int_{x\in \Xcal}\pi_{\parameter}(x)dx=1$, it holds that
    \begin{equation}
        \label{eq:gradzero}
        \beta\int_{x\in \Xcal}\nabla_{\parameter} \pi_{\parameter}(x)=\beta\ex{x\sim\pi_{\parameter}}{\nabla_{\parameter}\log\pi_{\parameter}(x)}=\ex{x\sim\pi_{\parameter}}{\nabla_{\parameter}f_{\parameter}(x)}=0.
    \end{equation} 
    %then we have $\ex{x\sim\pi_{\parameter}}{\nabla_{\parameter}f_{\parameter}(x)}=0$.
    
    % \item Entropy minimization. $f_{\parameter}(x)=\log \pi_{\parameter}(x)$
    \subsubsection{Variational Inference} Approximating complex distributions constitutes a significant facet of the Bayesian statistics. Compared to directly handling distributions within high-dimensional spaces, variational inference (VI) seeks to circumvent the computational complexities by approximating the target distribution with a parameterized simpler distribution \cite{jordan1999introduction,ranganath2014black,kingma2013auto,blei2017variational}. This approach is particularly advantageous when the exact posterior distribution is either too complex to be analytically solvable or when it is computationally prohibitive to sample from, as is often the case in high-dimensional Bayesian inference problems. VI operates on the principle of optimization, where the objective is to minimize the discrepancy between the approximating variational distribution $\pi_{\parameter}(x|y)$ and the true posterior $p(x,y)$. This discrepancy is frequently quantified using the Kullback-Leibler (KL) divergence, a non-symmetric measure of the difference between two probability distributions:
    \begin{equation}
        \label{eq:VI}
        \min_{\parameter\in \Theta}~ \mathrm{KL}(p(x,y)||\pi_{\parameter}(x|y)):= \Ebb_{x\sim\pi_{\parameter}(\cdot|y)}\Bigg[\underbrace{\log\left(\frac{p(x,y)}{\pi_{\parameter}(x|y)} \right)}_{f_\parameter(x)}\Bigg]. 
    \end{equation}
    Analogous to \eqref{eq:gradzero}, we have that $\ex{x\sim\pi_{\parameter}}{\nabla_{\parameter}f_{\parameter}(x)}=0$ for $f_{\parameter}(x)$ defined in \eqref{eq:VI}.

    There are also a lot of attention to improving the explicit variational distribution in VI by applying MCMC sampling. Salimans et al. (2015) propose a hybrid approach that utilizes MCMC within the variational framework to refine the variational distribution \cite{salimans2015markov}. By employing MCMC transitions that leave the target posterior distribution invariant, they introduce a rich class of variational distributions that can potentially capture complex dependencies and multimodal structures. A continuous relaxation is introduced by Maddison et al. (2017) for discrete random variables \cite{maddison2017filtering}, which facilitates the application of gradient-based optimization techniques. Their paper is particularly relevant when dealing with discrete latent variables in VI, where MCMC sampling can be leveraged to approximate the gradients of the objective function more effectively. Li et al. (2017) propose utilizing MCMC sampling not just as a means of refining a given variational distribution, but also as a mechanism for learning an efficient proposal distribution for variational inference \cite{li2017approximate}. By amortizing the cost of MCMC sampling across different iterations of the inference process, they aim to reduce the computational overhead while maintaining the benefits of MCMC sampling efficacy.

    \subsection{Gradient approximation using MCMC}

    To solve the stochastic optimization problem \eqref{eq:loss}, a common approach is to compute the stochastic gradient for iterations. 
    As the Assumption \ref{asm:loss} holds, the gradient of objective function $\Lcal(\parameter)$ is derived by
    \begin{equation}
        \label{eq:gradient}
        \begin{aligned}
        g(\parameter)&=\ex{x\sim \pi_{\parameter}}{f_{\parameter}(x)\nabla_{\parameter}\log \pi_{\parameter}(x)}+\ex{x\sim \pi_{\parameter}}{\nabla_{\parameter}f_{\parameter}(x)}\\
        &=\ex{x\sim \pi_{\parameter}}{\left(f_{\parameter}(x)-\ex{x\sim \pi_{\parameter}}{f_{\parameter}(x)}\right)\nabla_{\parameter}\phi_{\parameter}(x)}.
        \end{aligned}
    \end{equation} 
    The second equality is due to $\ex{x\sim\pi_{\parameter}}{\nabla_{\parameter}f_{\parameter}(x)}=0$ and $\nabla_{\parameter}\log \pi_{\parameter}(x)=\nabla_{\parameter}\phi_{\parameter}(x)-\ex{x\sim \pi_{\parameter}}{\nabla_{\parameter}\phi_{\parameter}(x)}$, implied by Assumption \ref{asm:loss}. We notice that the gradient can be expressed in the form of an expectation, thus allowing the approximation of the gradient through the Monte Carlo method. As for the unnormalized distribution $\pi_{\theta}(x)\propto \phi_{\parameter}(x)$, it is natural to employ MCMC methods to estimate the expectations presented within gradients. 

    Let us consider a Markov chain $\{X_i\}_{i=1}^{\infty}$ with a stationary distribution $\pi_{\parameter}(x)\propto e^{\phi_{\parameter}(x)}$. We have obtained a set of samples $S=\{x_i\}_{i=n_0+1}^{n_0+n}$ following $n_0$ burn-in steps. These samples are generated to approximate expectations with respect to $\pi_{\parameter}(x)$, which is typically intractable. The objective function and its gradient is estimated by the MCMC samples $S$:
    \begin{align}
        \hat{\Lcal}(\parameter;S)&=\frac{1}{|S|}\sum_{x\in S}f_{\parameter}(x),\\
        \hat{g}(\parameter;S)&=\frac{2}{|S|}\sum_{x\in S}\big(f_{\parameter}(x)-\frac{1}{|S|}\sum_{x\in S}f_{\parameter}(x)\big)\nabla_{\parameter} \phi_{\parameter}(x).\label{eq:approxgrad}
    \end{align}
    
    While MCMC methods offer a tractable and efficient strategy for approximating expectations, it must be noted that estimators derived from finite MCMC samples are generally biased. Furthermore, the samples are correlated, which affects the variance of the estimator. To mitigate bias, one must ensure a sufficiently large sample size. However, an excessively large sample size may increase computational burden and reduce the stochasticity for escaping local regions. In light of these considerations, the practitioner must carefully balance the sample size to minimize bias while preserving randomness and computational tractability. A rigorous approach may involve adaptive MCMC methods that adjust the burn-in period and sample size based on convergence diagnostics of Markov chains. It suggests that we should pay attention to the influence of sampling methods in optimization. 
    
    The SGD algorithm is utilized to update the parameters with MCMC samples:
\begin{equation}
	\label{eq:iter}
	\parameter_{k+1}= \parameter_{k} - \alpha_k\hat{g}(\parameter_k,S_{k}),
\end{equation}
where $\alpha_k$ are chosen stepsizes and $S_k$ is the obtained samples with the stationary distribution $\pi_{\parameter_{k}}$. MCMC-SGD provides a powerful framework for optimization in settings where the objective function is defined in terms of expectations over complex distributions. A detailed analysis on the interplay between SGD and MCMC sampling will aid us in deepening our understanding of convergence and refining the algorithms.

\begin{algorithm}[H] %算法开始 
	\caption{MCMC-SGD}
	\begin{algorithmic}[1]
		\label{alg:VMC}
		\REQUIRE The initialized parameter $\parameter_0$, the sample size $n$ and the length of burn-in periods $n_0$.
		\FOR{ $k=0,1,2,\dots$}
		\STATE Draw $n$ samples $S_k=\{\bmx_{n_0+1}^{(k)},\dots,\bmx^{(k)}_{n_0+n}\}$ by the MCMC algorithm after a burn-in period of $n_0$. 
		\STATE Compute the estimated gradient $\hat{g}(\parameter_k,S_{k})$ by \eqref{eq:approxgrad}.
		\STATE Update the parameter by \eqref{eq:iter} with the stepsize $\alpha_k$.
		\ENDFOR
	\end{algorithmic}
\end{algorithm} 

% discussion about convergence of VMC

\section{Convergence analysis of MCMC-SGD}
\label{sec:conv}

\setcounter{equation}{0}
In this section, we explore how SGD performs when it is coupled with MCMC samples. We begin by stating a number of assumptions that are crucial to guaranteeing the necessary characteristics of both the objective function and the sampling process. Then, the MCMC estimator is analyzed by the concentration inequality specific to Markov chains. This approach is fundamentally distinct from that employed in traditional SGD settings.
Upon establishing this analytical framework, we present our convergence theorem for the MCMC-SGD. This theorem reveals that, within a practical and flexible situation, the expected value of the gradient norm is assured to converge to zero. This convergence is noteworthy since it persists despite the bias inherent in MCMC sampling methods, highlighting the effectiveness of the algorithm under a variety of conditions.  It enhances our theoretical understanding of the integration of MCMC-SGD and offers guidance for improving algorithmic methodologies for various stochastic optimization problems.

\subsection{Assumptions}
\label{subsec:asm}
To lay the groundwork for our subsequent assumptions, we first expound upon the definition of sub-exponential random variables. This class of random variables is characterized by their tail behavior, which, in a certain sense, is less extreme than that of heavy-tailed distributions. A precise mathematical definition is as follows:
\begin{definition}
    The sub-exponential norm of a random variable $X$, which assumes values in a space $\mathcal{X}$ and follows a distribution $\pi$, is defined by the expression 
    \begin{equation*}
        \norm{X}_{\psi_1(\pi;\mathcal{X})}=\inf \left\{t > 0 : \ex{\pi}{\exp\lrp{\frac{|X|}{t}}}\leq 2\right\}.
    \end{equation*}
    In the event that $\norm{X}_{\psi_1(\pi;\mathcal{X})}$ is bounded above by a finite value, $X$ is called sub-exponential with parameter $\norm{X}_{\psi_1(\pi;\mathcal{X})}$. Moreover, the sub-exponential norm of a measurable function $f$ with respect to a distribution $\pi$ is defined by 
    \begin{equation*}
        \norm{f}_{\psi_1(\pi;\mathcal{X})}=\inf \left\{t > 0 : \ex{\pi}{\exp\lrp{\frac{|f(X)|}{t}}}\leq 2\right\}.
    \end{equation*}
\end{definition}

We introduce some regularity conditions on $f_{\parameter}$. Noticing that $f_{\parameter}$ may be unbounded for a general class problems, we give the sub-exponential assumption of $f_{\parameter}$ under the distribution $\pi_{\parameter}$.
\begin{assumption}
    \label{asm:local}
    Let $f_{\parameter}(x)$ given in \eqref{eq:loss} satisfy the following conditions. There exist constants $M,L_2>0$ such that
    \begin{enumerate}
        \setlength{\itemsep}{2pt}
        \item[(1)] $\sup\limits_{\parameter\in \Theta}\norm{f_{\parameter}}_{\psi_1(\pi_{\parameter};\mathcal{X})}\leq M$,
        \item[(2)] $\sup\limits_{\parameter\in \Theta}\Ebb_{\pi_{\parameter}}\lb \norm{\nabla_{\parameter}f_{\parameter}(x)}^2\rb\leq L_2^2$.
    \end{enumerate}
\end{assumption}

The mentioned two assumptions are both reasonable and encompass a broad class of scenarios. A common example is $f_{\parameter}=-r(x)+\beta\log \pi_{\parameter}(x)$ with $|r(x)|\leq R$ for all $x\in \Xcal$, which often appears in the optimization of entropy regularization and KL divergence. In this situation, $\norm{f_{\parameter}}_{\psi_1(\pi_{\parameter};\mathcal{X})}\leq t$ is equivalent to 
\begin{equation*}
    \begin{aligned}
        \ex{\pi_{\parameter}}{\exp\left(\frac{|-r(x)+\beta \log\pi_{\parameter}(x)|}{t}\right)}&\leq\ex{\pi_{\parameter}}{\exp\left(\frac{R-\beta \log\pi_{\parameter}(x)}{t}\right)} \\
        &=\int_{x\in \Xcal}e^{R/t}\pi_{\parameter}(x)^{1-\beta/t}dx\leq 2.
    \end{aligned}
\end{equation*}
Since $\int_{x\in \Xcal}\pi_{\parameter}dx=1$ holds for the distribution $\pi_{\parameter}$, there must exist a constant $t_{\parameter}$ relying $\parameter$ such that $\int_{x\in \Xcal}e^{R/t_{\parameter}}\pi_{\parameter}(x)^{1-\beta/t_{\parameter}}dx \leq 2$ and then Assumption \ref{asm:local} (1) is easily satisfied for a compact parameter space $\Theta$. Meanwhile, $\Ebb_{\pi_{\parameter}}\lb \norm{\nabla_{\parameter}f_{\parameter}(x)}^2\rb=\beta^2\ex{\pi_{\parameter}}{\norm{\nabla_{\parameter}\phi_{\parameter}(x)-\ex{\pi_{\parameter}}{\nabla_{\parameter}\phi_{\parameter}(x)}}^2}$, which can be bounded through Assumption \ref{asm:wavefun} (1) on the variational distribution $\pi_{\parameter}$. By establishing the reasonableness of the  Assumption \ref{asm:wavefun} (1) on the variational distribution, one can deduce the validity of the Assumption \ref{asm:local} (2) in this setting. For another instances, when solving the many-electron Schr\"{o}dinger equations, people often use the Jastrow factor \cite{jastrow1955many} to enable the network to efficiently capture the  cusps and decay of the wavefunction.
By exploiting physical knowledge in the construction of the wavefunction ansatz can make the function $f_{\parameter}$ smooth with respect to the parameters $\parameter$, such that Assumption \ref{asm:local} is satisfied in practical VMC calculations.

There are also some regularity conditions on the variational distribution $\pi_{\parameter}(x)$, in order to guarantee Lipschitz continuity of the gradient $g(\parameter)$. Typically, it is more practical to apply these smoothing measures to the energy function $\phi_{\theta}(x)$ rather than directly to the distribution $\pi_{\theta}(x)$. This kind of smoothing helps us control how the gradient changes, which is important for making sure our optimization methods work efficiently and reliably.
\begin{assumption}
    \label{asm:wavefun}
    Let $\phi_{\parameter}(x)$ be differentiable with respect to the parameters $\parameter\in \Theta$ for any $x\in \Xcal$. There exist constants $B,L_1>0$ such that
    \begin{enumerate}
        \setlength{\itemsep}{2pt}
        \item[(1)] $\sup\limits_{\parameter\in \Theta}\sup\limits_{x \in \Xcal}\norm{\nabla_{\parameter}\phi_{\parameter}(x)}\leq B$,
        \item[(2)] $\sup\limits_{\parameter\in \Theta}~\Ebb_{\pi_{\parameter}}\lb \norm{\nabla_{\parameter}^2\phi_{\parameter}(x)}^2\rb\leq L_1^2$.
    \end{enumerate}
\end{assumption}

The reasonableness of these two assumptions is solely based on the choice of parameterization. For instance, $\pi_\theta$ is taken as a canonical form of exponential family, i.e., $\phi_{\parameter}(x)=\parameter^{\T}T(x)$, where $\parameter$ represents the natural parameter vector and $T(x)$ is the sufficient statistic. As $\nabla_{\parameter}\phi_{\parameter}(x)=T(x)$ and $\nabla_{\parameter}^2\phi_{\parameter}(x)=0$, the Assumption \ref{asm:wavefun} holds when $T(x)$ is uniformly bounded for all $x\in \Xcal$. If $\phi_{\parameter}$ is parameterized by a complex network, it is equally imperative to ensure that its gradients with respect to any input and parameters are bounded.

\subsection{Analysis of the MCMC error}
\label{subsec:error}
%% Markov chain definition

We introduce following notations which are used frequently throughout this paper. For any function $f:\Xcal\rightarrow \Rbb$ and any distribution $\pi$ on $\Xcal$, we write its expectation $\Ebb_\pi[f]:=\int f(x)\pi(dx)$ and $p$-th central moment $\sigma^p_p[f]:=\Ebb[(f-\Ebb_\pi[f])^p]$. The second central moment, called the variance, is denoted by $\mathrm{Var}_{\pi}[f]=\sigma_2^2[f]:=\Ebb[(f-\Ebb_\pi[f])^2]$. 

Before the analysis of MCMC methods, we provide several general definitions about Markov chains. 
% and $L_{0}^2(\pi)=\{f\in L^{2}(\pi):\Ebb_\pi[f]=0 \}$ be its subspace of $\pi$-measure zero functions. 
Within our consideration, the state space $\Xcal$ is Polish and equipped with its $\sigma$-algebra $\Bcal$. Let $\{X_i\}_{i=1}^{n}$ be a time-homogeneous Markov chain defined on $\Xcal$. The distribution of the Markov chain is uniquely determined by its initial distribution $\nu$ and its transition kernel $P$. For any Borel set $A\in \Bcal$, let 
\begin{equation*}
    \begin{aligned}
        \nu(A)=\Pbb(X_1\in A),\quad P(X_i,A)=\Pbb(X_{i+1}\in A|X_i).
    \end{aligned}
\end{equation*}
A distribution $\pi$ is called stationary with respect to a transition kernel $P$ if 
\begin{equation*}
    \pi(A)=\int P(x,A)\pi(dx), ~\forall A\in \Bcal.
\end{equation*}
When the initial distribution $\nu=\pi$, we call the Markov chain stationary.

Our analysis starts from the perspective of operator theory on Hilbert spaces. Let $\pi$ be the stationary distribution of a Markov chain and $L^{2}(\pi)=\{f:\Ebb_\pi[f^2]<\infty \}$ be the Hilbert space equipped with the norm $\norm{f}_{\pi}=(\Ebb_\pi[f^2])^{1/2}$. Each transition kernel can be viewed as a Markov operator on the Hilbert space $L^2(\pi)$. The Markov operator $\operatorname{P}:L^{2}(\pi)\rightarrow L^{2}(\pi)$ is defined by 
\begin{equation*}
    \begin{aligned}
        \operatorname{P}f(x)=\int f(y)P(x,dy),~\forall x\in \Xcal,~\forall f \in L^{2}(\pi).
    \end{aligned}
\end{equation*}
It is easy to show that $\operatorname{P}$ has the largest eigenvalue $1$. Intuitively but not strictly, the gap between $1$ and other eigenvalues matters to the Markov chain from non-stationarity towards stationarity. Hence, we introduce the definition of the absolute spectral gap.
\begin{definition}[absolute spectral gap]
    \label{def:absgap}
     A Markov operator $\operatorname{P}:L^{2}(\pi)\rightarrow L^{2}(\pi)$ admits an absolute spectral gap $\gamma$ if 
    \begin{equation*}
        \gamma(\operatorname{P})=1-\lambda(\operatorname{P}):=1-\interleave \operatorname{P}-\Pi\interleave _{\pi}>0,
    \end{equation*}
    where  $\Pi:f\in L^2(\pi)\rightarrow\Ebb_\pi[f]\mathit{1}$ is the projection operator with $\mathit{1}$ denoting the identity operator and $\interleave\cdot\interleave_{\pi}$ is the operator norm induced by $\norm{\cdot}_{\pi}$ on $L^2(\pi)$.
\end{definition}
% \begin{definition}[right spectral gap]
%     \label{def:rsgap}
%     A Markov operator $P$ admits a right spectral gap $1-\gamma$ if 
%     \begin{equation*}
%         \gamma=\gamma(R):=\sup\{\langle Rf,f\rangle:f\in \Lcal^2_{0}(\pi)\},~~ where ~ R=(P+P^*)/2.
%     \end{equation*}
% \end{definition}
% By these definitions, we have
% \begin{equation}
%     \label{eq:gap}
%     |\gamma(R)|\leq \gamma(R)\leq \frac{\gamma(P)+\gamma(P^{*})}{2}=\gamma(P).
% \end{equation}

% \begin{remark}
%     For the finite-state time-homogeneous Markov chain, $\gamma(P)=\frac{\max\{|\gamma_2|,|\gamma_{min}|\}+1}{2}$ where $\gamma_2$ is actually the second largest eigenvalue of the transition matrix $P$ and $\gamma_{min}$ is the minimum eigenvalue. 
% \end{remark}

We consider the sample set $S=\{x_{i}\}_{i=n_0+1}^{n_0+n}$  generated from the desired distribution $\pi_{\parameter}$ by the MCMC algorithm . The Markov operator, denoted by $\operatorname{P}_{\parameter}$, determines the convergence rate of the Markov chain. We assume that there exists a uniform lower bound of the spectral gap.  
\begin{assumption}
    \label{asm:unigap}
    Let $\operatorname{P}_{\parameter}$ be the Markov operator induced by the MCMC algorithm with the absolute spectral gap $\gamma(\operatorname{P}_{\parameter})$. For any $\parameter\in \Theta$, there is a positive lower bound of absolute spectral gaps, that is,
    \begin{equation*}
        \gamma := \inf_{\parameter\in\Theta}\gamma(\operatorname{P}_{\parameter}) >0.
    \end{equation*}
\end{assumption}

This assumption excludes the situation that $\inf_{\parameter\in\Theta}\gamma(\operatorname{P}_{\parameter})=0$, which ensures the uniform convergence of the Markov chain.. The spectral gap $\gamma$ measures the lower bound on the mixing rate of the Markov chain, thereby guaranteeing the ergodicity of the MCMC algorithm. Compared to the uniformly geometric ergodicity assumptions commonly found in other articles \cite{karimi2019non,atchade2017perturbed}, the spectral gap assumption is significantly tighter, intuitively reflecting the factor that influence the convergence rate of a Markov chain. The spectral gap of the MCMC algorithm has been studied for some specific examples. On finite-state spaces, $\operatorname{P}$ becomes a transition matrix while $1-\gamma(\operatorname{P})$ relates to its eigenvalue. A survey of spectrums for Markov chains on discrete state-spaces is in \cite{saloff1997lectures}. This develops amounts of analytic techniques and \cite{diaconis1998we} has further applications to the Metropolis-Hastings algorithm. For the continuous spaces, there are few examples of sharp rates of convergence for the Metropolis-Hastings algorithm. In \cite{kienitz2000convergence,miclo2000trous}, it is claimed that the spectral gap $\gamma\sim O(\epsilon^2)$ for the Gaussian proposal $Q(x^{\prime}|x)\sim \exp\left(-\tfrac{1}{2\epsilon^2}\norm{x^{\prime}-x}^2\right)$.

We commence by stating Theorem \ref{thm:gbiasVar}, which articulates the core result that analyzes the error of MCMC gradient estimators. The proof of this theorem will rely on a series of auxiliary propositions, including Lemmas \ref{lem:Bern}, \ref{lem:Bern-exp} and \ref{lem:BernBiasVar}. Lemma \ref{lem:Bern} introduces a Bernstein inequality for general Markov chains. Using the Bernstein inequality, we obtain the tail bound of the MCMC estimators for sub-exponential functions in Lemma \ref{lem:Bern-exp}. Then, a detailed and intuitive error bound for the MCMC estimators is of consideration in Lemma \ref{lem:BernBiasVar}. Finally, through Lemma \ref{lem:BernBiasVar}, we succinctly capture the sampling error inherent in MCMC algorithms. In the following theorem, we delineate the error bounds of this stochastic gradient, which is one of main results in this paper.
\begin{theorem}
    \label{thm:gbiasVar}
    Let Assumptions \ref{asm:loss}, \ref{asm:local}, \ref{asm:wavefun} and \ref{asm:unigap} hold. For a fixed parameter $\parameter\in \Theta$, we generate the MCMC samples $S=\{x_i\}_{i=n_0+1}^{n_0+n}$ with the stationary distribution $\pi_{\parameter}$ by the MCMC algorithm. Suppose  we start from the initial distribution $\nu$ and let $\rchi=\rchi(\nu,\pi_{\parameter})<+\infty$. The stochastic gradient $\hat{g}(\parameter;S)$, defined by \eqref{eq:approxgrad} has the following error bounds that
    \begin{equation}
        \label{eq:BV}
        \begin{aligned}
            \norm{\Ebb[\hat{g}(\parameter;S)]-g(\parameter)}&\leq B_{n,n_0}:=\frac{4c_1B}{n\gamma},\\
            \epct{\norm{\hat{g}(\parameter;S)-g(\parameter)}^2}
            &\leq V_{n,n_0}:=\frac{16c_2B^2\sigma^2_{2}[f_{\parameter}]}{n\gamma}+\frac{40(c_3+c_4\log^4 n)B^2M^2}{n^2\gamma^2},
        \end{aligned}
    \end{equation}
    where with $\rchi_{n_0}= (1-\gamma)^{n_0}\rchi$ and $C=(1+\rchi_{n_0})^{\frac{1}{2}}$, these factors are defined by $c_1=\rchi_{n_0}\sigma_{2}[f_{\parameter}] +4M[\log \rchi_{n_0}]^{2}_{+}+4M[\log \rchi_{n_0}]_{+}$, $c_2=64(1+\log 2C)$ and $c_3=100(16+4\log 2C)^4$, $c_4=200$.
    % where $c_1=(1-\gamma)^{n_0}\rchi\sigma_{2}[E_{\parameter}] +4M[\log \rchi + n_0\log( 1-\gamma)]^{2}_{+}+4M[\log \rchi + n_0\log( 1-\gamma)]_{+}$, $c_2=128(3+\log C)$, $ c_3=100(8+4\log C)^4$ and $C=(1+(1-\gamma)^{2n_0}\rchi^2)^{\frac{1}{2}}$.
\end{theorem}
\begin{proof}
    The error of the stochastic gradient can be rewritten as 
    \begin{equation*}
        \begin{aligned}
            \norm{\Ebb[\hat{g}(\parameter;S)]-g(\parameter)}&=\sup_{\norm{v}=1}\abs{\Ebb[v^{\T}\hat{g}(\parameter;S)]-v^{\T}g(\parameter)},\\
            \epct{\norm{\hat{g}(\parameter;S)-g(\parameter)}^2}&=\tr\left(\epct{\left(\hat{g}(\parameter;S)-g(\parameter)\right)\left(\hat{g}(\parameter;S)-g(\parameter)\right)^{\T}}\right),\\
            &\leq d \sup_{\norm{v}=1}\Ebb[\left(v^{\T}\hat{g}(\parameter;S)-v^{\T}g(\parameter)\right)^2].
        \end{aligned}
    \end{equation*}
    For any given $v$ such that $\norm{v}=1$, $v^{\T}g(\parameter)$ is approximated by $v^{\T}\hat{g}(\parameter;S)$. We denote the stationary variables $F = f_{\parameter}(X)$, $Y = v^{\T}\nabla_{\parameter}\phi_{\parameter}(X)$, $\bar{F}= f_{\parameter}(X)-\Ebb_{x\sim \pi_{\parameter}} [f_{\parameter}(x)]$ with $X\sim \pi_{\parameter}$ and the empirical variables  $F_i = f_{\parameter}(x_i)$, $Y_i = v^{\T}\nabla_{\parameter}\phi_{\parameter}(x_i)$, $\bar{F}_i=F_i-\Ebb_{x\sim \pi_{\parameter}} [f_{\parameter}(x)]$. Then it holds 
    \begin{equation}
        \label{eq:gradsplit}
        \begin{aligned}
            \frac{1}{2}\left(v^{\T}\hat{g}(\parameter;S)-v^{\T}g(\parameter)\right)&=\frac{1}{n}\sum_{i=n_0+1}^{n+n_0}F_i Y_i-\left(\frac{1}{n}\sum_{j=n_0+1}^{n+n_0}F_j \right)\left(\frac{1}{n}\sum_{i=n_0+1}^{n+n_0}Y_i\right) - \Ebb_{\pi_{\parameter}} [\bar{F}Y]\\
            &=\underbrace{\frac{1}{n}\sum_{i=n_0+1}^{n+n_0}\bar{F}_i Y_i-\Ebb_{\pi_{\parameter}}[\bar{F}Y]}_{I_1}-\underbrace{\left(\frac{1}{n}\sum_{j=n_0+1}^{n+n_0}\bar{F}_i\right)\left(\frac{1}{n}\sum_{i=n_0+1}^{n+n_0}Y_i\right)}_{I_2}.
        \end{aligned}
    \end{equation}
    With Assumptions \ref{asm:wavefun} and \ref{asm:local}, we have $\norm{\bar{F}}_{\psi_1(\pi_{\parameter},\Xcal)}\leq M$ and $\norm{Y}\leq\norm{v}\norm{\nabla_{\parameter}\log\Psi_{\parameter}(x)} \leq B$. Then the variance is bounded by $\sigma_{2}^{2}[\bar{F}Y]\leq \Ebb[(\bar{F}Y)^2]\leq \sigma_2^{2}[F]B^2$. It also holds 
    \begin{equation*}
        \begin{aligned}
            \Ebb\left[\exp\left( \frac{\abs{\bar{F}Y-\Ebb[\bar{F}Y]}}{2MB}\right) \right]\leq  \Ebb\left[\exp\left( \frac{\abs{\bar{F}}B+MB }{2MB}\right) \right]\leq 2 ,
        \end{aligned}
    \end{equation*}
    which implies $\norm{\bar{F}Y}_{\psi_1(\pi_{\theta},\Xcal)}\leq 2MB$. 
    Applying Lemma \ref{lem:BernBiasVar} to $\{\bar{F}_i Y_{i}\}_{ i= n_0+1}^{n_0+n}$, we have %\cfn{more details? The factor $\sigma_{\parameter}$ in $c_1$ seems to be $\sqrt{d}\sigma_{\parameter}$?}
    \begin{equation}
        \begin{aligned}
            \label{eq:Bias}
            \norm{\Ebb[I_1]}\leq \frac{c_1B}{n\gamma},\quad
            \Ebb\left[ \norm{I_1}^2\right] \leq \frac{c_2dB^2\sigma^2_{2}[F_{\parameter}]}{n\gamma}+\frac{4(c_3d+c_4d\log^4 n)M^2B^2}{n^2\gamma^2}.
        \end{aligned}
    \end{equation} 
    % where $c_1=(1-\gamma)^{n_0}\rchi\sigma_{2}[E_{\parameter}] +4M[\log \rchi + n_0\log( 1-\gamma)]^{2}_{+}+4M[\log \rchi + n_0\log( 1-\gamma)]_{+}$, $c_2=128(2+\log C)$, $ c_3=100(8+4\log C)^4$ and $C=(1+(1-\gamma)^{2n_0}\rchi^2)^{\frac{1}{2}}$.

    As $\norm{\bar{F}}_{\psi_1(\pi_{\parameter},\Xcal)}\leq M$, $\Ebb_{\pi_{\parameter}}[\bar{F}]=0$ and $\norm{\frac{1}{n}\sum_{i=n_0+1}^{n+n_0}Y_i}\leq B$, Lemma \ref{lem:BernBiasVar} implies that
    \begin{equation}
        \begin{aligned}
            \label{eq:Var}
            \norm{\Ebb[I_2]}\leq \frac{c_1B}{n\gamma},\quad \Ebb\left[ \norm{I_2}^2\right] \leq \frac{c_2dB^2\sigma^2_{2}[F_{\parameter}]}{n\gamma}+\frac{(c_3d+c_4d\log^4 n)M^2B^2}{n^2\gamma^2}.
        \end{aligned}
    \end{equation}
    Combining \eqref{eq:Bias} and \eqref{eq:Var}, we finally obtain
    \begin{equation*}
        \begin{aligned}
            \norm{\Ebb[\hat{g}(\parameter;S)]-g(\parameter)}&=2\norm{\Ebb[I_1-I_2]}\leq 2\norm{\Ebb[I_1]}+2\norm{\Ebb[I_2]}\leq B_{n,n_0},\\
            \epct{\norm{\hat{g}(\parameter;S)-g(\parameter)}^2}
            &=4\Ebb[\norm{I_1-I_2}^2]\leq 8\Ebb[\norm{I_1}^2]+8\Ebb[\norm{I_2}^2]\leq V_{n,n_0},
        \end{aligned}
    \end{equation*}
    where $B_{n,n_0},V_{n,n_0}$ are given in \eqref{eq:BV}.
\end{proof}

To refine the proof of the Theorem \ref{thm:gbiasVar}, we first introduce a pivotal tool in this paper, concentration inequalities, to analyze the error of the MCMC estimators in details. The essence of concentration inequalities lies in their capacity to quantify how a random variable can concentrate around its expected value. A Bernstein inequality for general Markov chains is proposed by Jiang and Fan \cite{jiang2018bernstein} and beneficial to our analysis. The following lemma is a direct corollary of Theorem 2 in \cite{jiang2018bernstein}.
%Then, Fan et al. (2021) derive a non-asymptotic error bound for MCMC estimation using Hoeffding's inequality \cite{fan2021hoeffding}. A corollary is obtained by these two references without much effort.

\begin{lemma}%[Jiang and Fan, 2018]
    \label{lem:Bern}
    Let $\{X_i\}_{i= 1}^{n}$ be a Markov chain with stationary distribution $\pi$ and absolute spectral gap $\gamma$. Suppose the initial distribution $\nu$ is absolute continuous with respect to the stationary distribution $\pi$ and its derivative $\tfrac{d\nu}{d\pi}\in L^2(\pi)$. Consider a bounded function $h:\mathcal{X}\rightarrow [-c,c]$ with $\Ebb_\pi[h]=0$ and variance $\sigma_2^2[h]$. Then, when $\nu=\pi$ , that is, $\{X_i\}_{i= 1}^{n}$ is stationary, it holds that
    \begin{equation}
        \label{eq:Bern-stat}
        \mathbb{P}_{\pi}\left(\frac{1}{n} \sum_{i=1}^{n} h\left(X_{i}\right)\geq s\right) 
    \leq \exp\left(-\frac{\gamma ns^2}{4\sigma^2+5cs}\right), \qquad \forall s\geq 0.
    \end{equation}
    % Therefore, for the general case, it holds that
    % \begin{equation}
    %     \label{eq:Bern}
    %     \mathbb{P}_{\nu}\left(\frac{1}{n} \sum_{i=1}^{n} f\left(X_{i}\right)\geq s\right) 
    % \leq (1+\rchi^2(\nu,\pi))^{\frac{1}{2}} \exp\left(-\frac{\gamma ns^2}{8\sigma^2+10cs}\right), \qquad \forall s\geq 0.
    % \end{equation}
    % where $\rchi^2(\nu,\pi):=\norm{\tfrac{d\nu}{d\pi}-1}^{2}_{\pi}$ represents the chi-squared divergence between $\nu$ and $\pi$.
\end{lemma}

The Bernstein inequality \eqref{eq:Bern-stat} shows how the average of MCMC samples concentrates at the expectation. However, the Markov chain is not always non-stationary. We define $\rchi^2(\nu,\pi):=\norm{\tfrac{d\nu}{d\pi}-1}^{2}_{\pi}$ represents the chi-squared divergence between $\nu$ and $\pi$, by which the Bernstein inequality can be extended into a non-stationary one. We abbreviate $\rchi=\rchi(\nu,\pi)$ and $C=(1+\rchi^2)^{\frac{1}{2}}$. Then, we derive the tail bound of the MCMC estimator for sub-exponential functions in the following lemma.

\begin{lemma}
    \label{lem:Bern-exp}
    Let $\{X_i\}_{i= 1}^{n}$ be a non-stationary Markov chain with stationary distribution $\pi$ and absolute spectral gap $\gamma$. Suppose the initial distribution $\nu$ is absolute continuous with respect to the stationary distribution $\pi$ and its derivative $\tfrac{d\nu}{d\pi}\in L^2(\pi)$.  We consider a function $h\in L^2(\pi)$ satisfying $ \norm{h}_{\psi_1(\pi,\Xcal)}\leq M$. If $s\geq \frac{10M(\log n)^2}{n}$, the following tail bound holds,  
    \begin{equation}
        \label{eq:tail}
        \mathbb{P}_{\nu}\left(\abs{ \frac{1}{n} \sum_{i=1}^{n} h\left(X_{i}\right)-\Ebb_\pi[h] }\geq s\right) 
    \leq 2C\exp\left(-\frac{\gamma ns^2}{64\sigma^2_2[h]}\right) + 2C\exp\left(-\sqrt\frac{\gamma ns}{160M}\right).
    \end{equation}
    In other words, for $\delta>0$, with probability at least $1-\delta$, it holds that
    \begin{equation}
        \label{eq:highprobbound}
        \left|\frac{1}{n}\sum_{i=1}^{n}h(X_i)-\Ebb_\pi[h]\right|\leq 8\sigma_2[h]\sqrt{\frac{\log(4C/\delta)}{n\gamma}}+ 160M\frac{[\log(4Cn/\delta)]^2}{n\gamma}.
    \end{equation}
\end{lemma}

\begin{proof}
    Without loss of generality, we assume that $\Ebb_\pi[h]=0$ in the following proof.
    To deal with a possibly unbounded $h$, we firstly fix a $M^{\prime}>0$ and consider the truncation function $\bar{h}=\max\{\min\{h,M^{\prime}\},-M^{\prime}\}$ and $\hat{h}=h-\bar{h}$. When $h$ is a sub-exponential random variable, its tail decays at least as fast as an exponential function. As $\norm{h}_{\psi_1(\pi,\Xcal)}\leq M$, the Markov's inequality implies 
    \begin{equation}
        \label{eq:subexp-tail}
        \Pbb_\pi(\abs{h}>s)\leq \exp\left(-\frac{s}{M}\right)\Ebb_\pi\left[\exp\left(\frac{|h|}{M}\right)\right] \leq 2\exp\left(-\frac{s}{M}\right).
    \end{equation}
    % Then, for any $p\geq 1$, we have
    % \begin{equation}
    %     \label{eq:pmonent-hatf}
    %     \begin{aligned}
    %         \Ebb_\pi\big[|\hat{f}|^p\big]&=\int_{0}^{+\infty} ps^{p-1}\Pbb_\pi(|\hat{f}|>s)ds\leq \int_{0}^{+\infty} ps^{p-1}\Pbb_\pi(|f|>s+M^{\prime})ds\\
    %         &\leq 2p \int_{0}^{+\infty} s^{p-1}\exp\left(-\frac{s+M^{\prime}}{M}\right)ds=2\cdot p! M^{p}\exp\left(-\frac{M^{\prime}}{M}\right).
    %     \end{aligned}
    % \end{equation}
    
    % By introducing the truncation $\bar{f},\hat{f}$, we make the corresponding decomposition,
    % \begin{align*}
    %     &\Delta:=\frac{1}{n}\sum_{i=1}^{n}f(X_i)=\bar{\Delta}+\hat{\Delta},\\
    %     &\bar{\Delta}:=\frac{1}{n}\sum_{i=1}^{n}\bar{f}(X_i)-\Ebb_\pi[\bar{f}],~~
    %     \hat{\Delta}:=\frac{1}{n}\sum_{i=1}^{n}\hat{f}(X_i)-\Ebb_\pi[\hat{f}].
    % \end{align*}
    % We analyze the error from two aspects.
    % \paragraph{High probability bound} 
    Notice that for any event $A\in\sigma(X_1,\cdots,X_n)$, it holds from the Cauchy-Schwarz inequality
    \begin{equation}
        \label{eq:non-stat}
        \begin{aligned}
            \Pbb_{\nu}(A)
            =&~ \int_{\mathcal{X}} \Pbb(A|X_1=x)\nu(dx) 
            = \int_{A} \Pbb(A|X_1=x)\frac{d\nu}{d\pi}\pi(dx)\\
            \leq&~ \sqrt{\int_{\mathcal{X}} \left(\frac{d\nu}{d\pi}\right)^2\Pbb(A|X_1=x)\pi(dx) \int_{\mathcal{X}} \Pbb(A|X_1=x) \pi(dx) } \\
            \leq&~ C\sqrt{\Pbb_{\pi}(A)}.
        \end{aligned}
    \end{equation}
    Hence, we only consider the case $\nu=\pi$, i.e., the case $\{X_i\}_{i=1}^{n}$ is stationary.
    We first fix a large $M^{\prime}>0$. For any $1\leq i\leq n$, we have 
    \begin{equation}
        \label{eq:highprob1}
        \begin{aligned}
            \Pbb\left(\abs{h(X_i)}>M^{\prime} \right) \leq 2\exp\left(-\frac{M^{\prime}}{M}\right).
        \end{aligned}
    \end{equation} 
    It follows from the inclusion of events that
    \begin{equation*}
        \begin{aligned}
            \mathbb{P}\left(\abs{ \frac{1}{n} \sum_{i=1}^{n} h\left(X_{i}\right) }\geq s\right)
            \leq&~\mathbb{P}\left(\abs{ \frac{1}{n} \sum_{i=1}^{n} \bar{h}\left(X_{i}\right) }\geq s\right) + \Pbb\left( \exists ~ 1\leq i\leq n,~ \abs{h(X_i)}>M^{\prime}\right)\\
            \leq&~\mathbb{P}\left(\abs{ \frac{1}{n} \sum_{i=1}^{n} \bar{h}\left(X_{i}\right) }\geq s\right) + + 2n\exp\left(-\frac{M^{\prime}}{M}\right).
        \end{aligned}
    \end{equation*}
    Notice that $\abs{\EE_{\pi}[\bar{h}]}=\abs{\EE_{\pi}[\hat{h}]}\leq 2M\exp\left(-\frac{M^{\prime}}{M}\right)$. Therefore, when $s\geq 2\abs{\EE_{\pi}[\hat{h}]}$, it holds that
    \begin{equation*}
        \begin{aligned}
            &\mathbb{P}\left(\abs{ \frac{1}{n} \sum_{i=1}^{n} \bar{h}\left(X_{i}\right) }\geq s\right) \\
            &\leq~\mathbb{P}\left(\abs{ \frac{1}{n} \sum_{i=1}^{n} \bar{h}\left(X_{i}\right) -\EE_{\pi}[\bar{h}]}\geq s-\abs{\EE_{\pi}[\bar{h}]}\right) \\
            &\leq~\mathbb{P}\left(\abs{ \frac{1}{n} \sum_{i=1}^{n} \bar{h}\left(X_{i}\right) -\EE_{\pi}[\bar{h}]}\geq \frac{s}{2}\right) \\
            &\leq~ 2\exp\left(-\frac{\gamma ns^2}{16\sigma^2_2[\bar{h}]+10M^{\prime}s}\right) \leq~ 2\exp\left(-\frac{\gamma ns^2}{16\sigma^2_2[h]+10M^{\prime}s}\right),
        \end{aligned}
    \end{equation*} 
    where the third inequality is due to \eqref{eq:Bern-stat}, and  the last inequality uses $\sigma^2_2[\bar{h}]\leq \sigma^2_2[h]$. 
    Finally, for any fixed $s\geq \frac{10M(\log n)^2}{n}$, we take $M^{\prime}=\sqrt{\gamma nMs/10}$ and then obtain

        \begin{align*}
            &\mathbb{P}\left(\abs{ \frac{1}{n} \sum_{i=1}^{n} h\left(X_{i}\right) }\geq s\right)\\
            &\leq~2\exp\left(-\frac{\gamma ns^2}{16\sigma^2_2[h]+10M^{\prime}s}\right) + 2n\exp\left(-\frac{M^{\prime}}{M}\right)\\
            % =&~ 2\exp\left(-\frac{\gamma ns^2}{16\sigma^2_2[h]+10s\sqrt{\gamma nMs/10}}\right) + 2n\exp\left(-\sqrt\frac{\gamma ns}{10M}\right)\\
            &\leq~ 2\exp\left(-\frac{\gamma ns^2}{2\max\{16\sigma^2_2[h],10s\sqrt{\gamma nMs/10}\}}\right) + 2n\exp\left(-\sqrt\frac{\gamma ns}{10M}\right)\\
            &\leq~ 2\exp\left(-\frac{\gamma ns^2}{32\sigma^2_2[h]}\right) + 2\exp\left(-\sqrt\frac{\gamma ns}{40M}\right) + 2n\exp\left(-\sqrt\frac{\gamma ns}{10M}\right)\\
            &\leq~ 2\exp\left(-\frac{\gamma ns^2}{32\sigma^2_2[h]}\right) + 4\exp\left(-\sqrt\frac{\gamma ns}{40M}\right),
        \end{align*}
    where the last inequality holds from $n\exp\left(-\sqrt\frac{\gamma ns}{10M}\right)\leq \exp\left(-\sqrt\frac{\gamma ns}{40M}\right)$ as long as $s\geq \frac{10M(\log n)^2}{n}$. Using the property \eqref{eq:non-stat}, it follows that the tail bound \eqref{eq:tail} holds.
    This completes the proof.
\end{proof}

In the above lemma, we establish the tail bound of the MCMC estimator for unbounded functions. To the best of our knowledge, this result has not appeared in the literatures of Bernstein inequality for Markov chains. The novelty of our method is manifested in the sub-exponential assumption tailored for unbounded functions and the utilization of spectral gap analysis within concentration inequalities. Consequently, a comprehensive and intuitive error bound for the MCMC estimator is provided as follows.

\begin{lemma}
    \label{lem:BernBiasVar}
    Suppose that the condition of Lemma \ref{lem:Bern-exp} holds. Then, error bounds for the MCMC estimator are given as follows. 
    
    (1) The bias satisfies that
    \begin{equation}
        \label{eq:generalbias}
        \left|\Ebb_{\nu}\lrb{\frac{1}{n}\sum_{i=1}^{n}h(X_i)}-\Ebb_\pi[h]\right|
        \leq\frac{c_1}{n\gamma},
    \end{equation}
    where $c_1=\sigma_2[h]\minop{1,\rchi}+4M\pos{\log\rchi}^2+4M\pos{\log\rchi}$ and $\rchi=\rchi(\nu,\pi)$. In particular, when $\rchi\leq 1$, $c_1=\sigma_2[h]\rchi$.

    (2) The variance satisfies that
    \begin{equation}
        \label{eq:generalvariance}
        \Ebb_{\nu}\lrb{\left|\frac{1}{n}\sum_{i=1}^{n}h(X_i)-\Ebb_\pi[h]\right|^2} \leq\frac{c_2}{n\gamma}\sigma_2^2[h]+\frac{c_3+c_4\log^4 n}{n^2\gamma^2}M^2,
    \end{equation}
    where $c_2=64(1+\log 2C)$, $c_3=100(16+4\log 2C)^4$, $c_4=200$, and $C=(1+\rchi^2)^{\frac{1}{2}}$. 
\end{lemma}

\begin{proof}
    % Without loss of generality, we assume that $\Ebb_\pi[f]=0$ in the following proof.
    % To deal with a possibly unbounded $f$, we firstly fix a $M'>0$ and consider the truncation function $\bar{f}=\max\{\min\{f,M^{\prime}\},-M^{\prime}\}$ and $\hat{f}=f-\bar{f}$. The basic property of $\hat{f}$ is that, as $\norm{f(X)}_{\psi_1}\leq M$, the Markov's inequality implies 
    % \begin{equation}
    %     \label{eq:subexp-tail}
    %     \Pbb_\pi(\abs{f}>s)\leq \exp\left(-\frac{s}{M}\right)\Ebb_\pi\left[\exp\left(\frac{|f|}{M}\right)\right] \leq 2\exp\left(-\frac{s}{M}\right).
    % \end{equation}
    % Then, for any $p\geq 1$, we have
    % \begin{equation}
    %     \label{eq:pmonent-hatf}
    %     \begin{aligned}
    %         \Ebb_\pi\big[|\hat{f}|^p\big]&=\int_{0}^{+\infty} ps^{p-1}\Pbb_\pi(|\hat{f}|>s)ds\leq \int_{0}^{+\infty} ps^{p-1}\Pbb_\pi(|f|>s+M^{\prime})ds\\
    %         &\leq 2p \int_{0}^{+\infty} s^{p-1}\exp\left(-\frac{s+M^{\prime}}{M}\right)ds=2\cdot p! M^{p}\exp\left(-\frac{M^{\prime}}{M}\right).
    %     \end{aligned}
    % \end{equation}
    
    % By introducing the truncation $\bar{f},\hat{f}$, we make the corresponding decomposition,
    % \begin{align*}
    %     &\Delta:=\frac{1}{n}\sum_{i=1}^{n}f(X_i)=\bar{\Delta}+\hat{\Delta},\\
    %     &\bar{\Delta}:=\frac{1}{n}\sum_{i=1}^{n}\bar{f}(X_i)-\Ebb_\pi[\bar{f}],~~
    %     \hat{\Delta}:=\frac{1}{n}\sum_{i=1}^{n}\hat{f}(X_i)-\Ebb_\pi[\hat{f}].
    % \end{align*}
    % We analyze the error from two aspects.
    We follow the notations in the proof of Lemma \ref{lem:Bern-exp}.
    \paragraph{Bias bound.} 
    Clearly, we only need to bound $\frac{1}{n}\sum_{i=1}^{n}\abs{\Ebb_{\nu_i}[h]-\Ebb_{\pi}[h]}$, where $\nu_i= \nu \oP^{i-1}$ is the distribution of $X_i$. 
    For the distribution $\pi^{\prime}$ whose derivative $\frac{d\pi^{\prime}}{d\pi}\in L^{2}(\pi)$, it holds from the Cauchy-Schwarz inequality that 
    \begin{equation}\label{eq:change-measure}
        \begin{aligned}
            \abs{\Ebb_{\pi^{\prime}}[h]-\Ebb_\pi[h]}
        &=\abs{ \EE_{\pi}\left[\left(\frac{d\pi^{\prime}}{d\pi}-1\right)(h-\Ebb_\pi[h])\right] }\\
        &\leq  \left[\EE_{\pi}\left(\frac{d\pi^{\prime}}{d\pi}-1\right)^2\right]^{1/2}\sigma_2[h]
        = \rchi(\pi^{\prime},\pi)\sigma_2[h].
        \end{aligned}
    \end{equation}
    Besides, noticing that $\left(\Ebb_\pi \abs{\hat{h}}^2\right)^{\frac12}\leq 2M\exp\left(-\frac{M^{\prime}}{2M}\right)$, we have
    \begin{align*}
        \abs{\Ebb_{\pi^{\prime}}[\hat{h}]-\Ebb_{\pi}[\hat{h}]}
        \leq \rchi(\pi^{\prime},\pi) \left(\Ebb_\pi \abs{\hat{h}}^2\right)^{\frac12}
        \leq 2M\rchi(\pi^{\prime},\pi)\exp\left(-\frac{M^{\prime}}{2M}\right).
    \end{align*}
    % \begin{align}\label{eq:change-measure}
    %     \abs{\nu(f)-\pi(f)}
    %     =\abs{ \EE_{\pi}\left[\left(\frac{d\nu}{d\pi}-1\right)f\right] }
    %     \leq \sigma_f \left[\EE_{\pi}\left(\frac{d\nu}{d\pi}-1\right)^2\right]^{1/2}
    %     =\sigma_f\rchi(\nu,\pi).
    % \end{align}
    Then \eqref{eq:change-measure} implies that
    \begin{equation}
        \label{eq:change-function}
        \begin{aligned}
            \abs{\Ebb_{\nu_i}[h]-\Ebb_{\pi}[h]}
        \leq&~ \abs{\Ebb_{\nu_i}[\bar{h}]-\Ebb_{\pi}[\bar{h}]}+\abs{\Ebb_{\nu_i}[\hat{h}]-\Ebb_{\pi}[\hat{h}]}\\
        \leq&~ 2M^{\prime}+2M\rchi(\nu_i,\pi)\exp\left(-\frac{M^{\prime}}{2M}\right),
        \end{aligned}
    \end{equation}
    and as long as $\rchi(\nu_i,\pi)\geq 1$, we can choose $M^{\prime}=2M\log \rchi(\nu_i,\pi)$ in \eqref{eq:change-function} to derive $\abs{\Ebb_{\nu_i}[h]-\Ebb_{\pi}[h]}\leq 4M(1+\log \rchi(\nu_i,\pi))$.
    Combining \eqref{eq:change-measure} and \eqref{eq:change-function} yields that
    \begin{equation}
        \begin{aligned}
            \label{eq:nu-pi}
            \abs{\Ebb_{\nu_i}[h]-\Ebb_{\pi}[h]}
        \leq \begin{cases}
             \rchi(\nu_i,\pi)\sigma_2[h], & \text{always},\\
            4M\left(1+\log\rchi(\nu_i,\pi)\right), &\text{if }\rchi(\nu_i,\pi)\geq 1.
        \end{cases}
        \end{aligned}
    \end{equation}
    Now, by the definition of $\gamma$, it holds that $\rchi(\nu_i,\pi)\leq (1-\gamma)^{i-1} \rchi(\nu_1,\pi)=(1-\gamma)^{i-1} \rchi$. 
        Let us consider the smallest $k\geq 1$ such that $\rchi(\nu_{k},\pi)\leq 1$. Then clearly $k\leq 1+\ceil{\frac{[\log\rchi]_+}{\gamma}}$, and for $i\geq k$, $\rchi(\nu_{i},\pi)\leq (1-\gamma)^{i-1}\min\{1,\rchi\}$. Hence,
        \begin{align*}
            \sum_{i=1}^{n}\abs{\Ebb_{\nu_i}[h]-\Ebb_{\pi}[h]}
            \leq& \sum_{i=1}^{k-1}\abs{\Ebb_{\nu_i}[h]-\Ebb_{\pi}[h]}+\sum_{i=k}^{n}\abs{\Ebb_{\nu_i}[h]-\Ebb_{\pi}[h]}\\
            \leq& \sum_{i=1}^{k-1}4M\left(1+\log\rchi(\nu_i,\pi)\right)+\sum_{i=k}^{n} \rchi(\nu_i,\pi)\sigma_2[h]\\
            \leq& 4M(k-1)\left(1+\log\rchi\right)+ \sigma_2[h] \min\{1,\rchi\}\cdot \frac1{\gamma},
        \end{align*}
    where the first inequality is the triangle inequality and the second uses \eqref{eq:nu-pi}. Therefore, it holds that
    \begin{equation*}
        \frac{1}{n}\sum_{i=1}^{n}\abs{\Ebb_{\nu_i}[h]-\Ebb_{\pi}[h]}
        \leq \frac{1}{n\gamma}\left(\sigma_2[h] \minop{
            1, \rchi}+4M\pos{\log\rchi}^2+4M\pos{\log\rchi}\right).
    \end{equation*}

    \paragraph{Variance bound.}
    In the following proof, we provide an upper bound on all higher-order moments of $\Delta$ simultaneously.
    Denote $A=\log(2C), R_1=8\sqrt{\frac{\sigma_2^2[h]}{\gamma n}}, R_2=\frac{160M}{\gamma n}, s_0=\frac{10M(\log n)^2}{n}$. Then, we only need to bound $\Ebb[\abs{\Delta}^m]$ for even $m$ under the condition
    \begin{align*}
        \Pbb\left(\abs{\Delta}\geq s\right)\leq \exp(A-s^2/R_1^2)+\exp(A-\sqrt{s/R_2}), \qquad \forall s\geq s_0.
    \end{align*}
    Notice that
    \begin{align*}
        \frac1m\Ebb[\abs{\Delta}^m]
        =&~\int_{0}^\infty s^{m-1}\Pbb(\abs{\Delta}\geq s)ds\\
        \leq&~\int_{0}^\infty s^{m-1}\min\{1,\exp(A-s^2/R_1^2)+\exp(A-\sqrt{s/R_2})\}ds\\
        \leq&~I_1+I_2.
    \end{align*}
    where $I_1= \int_{0}^\infty s^{m-1}\min\{1,\exp(A-s^2/R_1^2)\}ds$ and $I_2= \int_{0}^\infty s^{m-1}\min\{1,\exp(A-\sqrt{s/R_2})\}ds$.
    Thus, we further denote $s_1=\max\{s_0,R_1\sqrt{A}\}, s_2=\max\{s_0,R_2A^2\}$. Then
    \begin{align*}
        I_1=&~\int_{0}^{s_1} s^{m-1}ds+\int_{s_1}^\infty s^{m-1}\exp(A-s^2/R_1^2)ds\\
        =&~\frac{s_1^m}{m}+\frac{R_1^{m}}{2}\int_{s_1}^\infty \left(\frac{s}{R_1}\right)^{m-2}\exp(A-s^2/R_1^2)d\left(\frac{s^2}{R_1^2}\right)\\
        % =&~\frac{s_1^m}{m}+\frac{R_1^{m}}{2}\cdot\exp(A-s_1^2/R_1^2)\cdot \left(\frac{m-2}{2}\right)!,
        =&~\frac{s_1^m}{m}+\frac{R_1^{m}}{2}\int_{s_1^2/R_1^2}^\infty t^{(m-2)/2}\exp(A-t)dt\\
        \leq&~\frac{s_1^m}{m}+\frac{R_1^{m}}{m}\left((A+m/2)^{m/2}-A^{m/2}\right),
    \end{align*}
    where the last inequality is due to the fact $\int_{A}^{\infty} t^{k-1}\exp(A-t)dt\leq (A+k-1)^{k-1}\leq ((A+k)^k-A^k)/k$ for any integer $k\geq 1$. Similarly,
    \begin{align*}
        I_2=&~\int_{0}^\infty s^{m-1}\min\{1,\exp(A-\sqrt{s/R_2})\}ds\\
        =&~\int_{0}^{s_2} s^{m-1}ds+\int_{s_2}^\infty s^{m-1}\exp(A-\sqrt{s/R_2})ds\\
        =&~\frac{s_2^m}{m}+\int_{\sqrt{s_2/R_2}}^\infty 2R_2^{m}t^{2m-1}\exp(A-t)dt\\
        \leq&~ \frac{s_2^m}{m}+\frac{R_2^{m}}{m}\left((A+2m)^{2m}-A^{2m}\right).
    \end{align*}
    Combining these two cases, we obtain
    \begin{align*}
        \Ebb[\abs{\Delta}^m]\leq&~ mI_1+mI_2\\
        \leq&~ s_0^m+s_1^m+R_1^m\left((A+m/2)^{m/2}-A^{m/2}\right) + R_2^m \left((A+2m)^{2m}-A^{2m}\right)\\
        \leq&~ 2s_0^m+R_1^m(A+m/2)^{m/2} + R_2^m (A+2m)^{2m}.
    \end{align*}
    This completes the proof.
\end{proof}

\subsection{First order convergence}
In this subsection, we will first analyze the decrease for a single iteration based on the MCMC error bounds. Then, our main result delineates the convergence of the expected norm of the gradient throughout the iterations of SGD, when cooperating with the MCMC estimator. This convergence is established through a rigorous theoretical framework, underscoring the efficacy of the SGD algorithm in the presence of sampling randomness.

Under our assumptions in subsection \ref{subsec:asm}, we can affirm the $L$-smoothness of the objective function $\Lcal(\parameter)$ to facilitate a conventional optimization analysis. This is a common condition for functions to undergo expected descent. We prove the following lemma by giving the bound of the Hessian.
\begin{lemma}
    \label{prop:gLip}
    Let Assumptions \ref{asm:loss}, \ref{asm:local} and \ref{asm:wavefun} hold.  There exists a constant $L>0$ such that $\Lcal(\parameter)$ is $L$-smooth, that is
    \begin{equation}
        \label{eq:gLip}
        \|g(\parameter_1)-g(\parameter_2)\|\leq L\|\parameter_1-\parameter_2\|,~\forall \parameter_1,\parameter_2 \in \Theta.
    \end{equation} 
\end{lemma}

% \begin{proof}
%     A random variable $Y$ is of sub-Gaussian type if there exists $\sigma,c>0$ such that
%     \[\begin{aligned}
%         \Ebb\left[e^{tY}\right]\leq c\exp\left(\frac{\sigma^2t^2}{2}\right).
%     \end{aligned}\]
%     The definition above implies that $\Pbb\left(|Y|>t\right)\leq 2c\exp\left(-\frac{t^2}{2\sigma^2}\right)$, and thus
%     \[\begin{aligned}
%         \Ebb\left[Y^2\right]
%         &=\int_{0}^{\infty} \Pbb\left(Y^2\geq s\right)ds\\
%         &\leq \int_{0}^{\infty}\min\left\{1,2c\exp\left(-\frac{s}{2\sigma^2}\right)\right\}ds\\
%         &=2\sigma^2(\log 2c +1).
%     \end{aligned}\]
%     Combined with \eqref{eq:HoeffdingExp}, this fact implies that
%     \[\begin{aligned}
%         \Ebb\left[\frac{1}{n}\sum_{i=n_0+1}^{n_0+n}f(X_i)-\pi(f)\right]^2
%         \leq \frac{2\sigma^2(\log 2C +1)}{n} .
%     \end{aligned}\]
% \end{proof}
% Besides, a high probability upper bound is given by the following corollary.
% \begin{corollary}
%     \label{crl:highprob}
%     Let the conditions in Lemma \ref{lem:hoeffding} be satisfied. For any $\delta \in (0,1)$, with probability at least $1-\delta$,
%     \begin{equation}
%         \left|\frac{1}{n}\sum_{i=n_0+1}^{n_0+n}f(X_i)-\pi(f)\right|\leq \sqrt{\frac{2\sigma^2\log (2C/\delta)}{n} }.
%     \end{equation}
% \end{corollary}
% \begin{proof}
%     Let the right of the inequality \eqref{eq:HoeffdingProb} be $\delta$, then $\epsilon$ become the high probability upper bound. 
% \end{proof}

\begin{proof}
    For any $\parameter\in \Theta$, $\norm{f_{\parameter}}_{\psi_1(\pi_{\parameter};\mathcal{X})}\leq M$ implies that
    \begin{equation*}
        \Ebb_{\pi_{\parameter}}\lrb{|f_{\parameter}|^2}^{1/2}\leq\Ebb_{\pi_{\parameter}}\lrb{|f_{\parameter}|}\leq M. 
    \end{equation*}
    The Hessian  $H(\parameter):=\nabla_{\parameter}^2\Lcal(\parameter)$ can be derived by
    \begin{equation}
        \label{eq:Hess}
        \begin{aligned}
            H(\parameter)=&\ex{\pi_{\parameter}}{\nabla_{\parameter}\log\pi_{\parameter}\nabla_{\parameter}f_{\parameter}^{\T}}+\ex{\pi_{\parameter}}{f_{\parameter}\nabla_{\parameter}^2\log\pi_{\parameter}}+\ex{\pi_{\parameter}}{f_{\parameter}\nabla_{\parameter}\log\pi_{\parameter}\nabla_{\parameter}\log\pi_{\parameter}^{\T}},
        \end{aligned}
    \end{equation}
    where $\nabla_{\parameter}\log\pi_{\parameter}=\nabla_{\parameter}\phi_{\parameter}-\ex{\pi_{\parameter}}{\nabla_{\parameter}\phi_{\parameter}}$ and $\nabla_{\parameter}^2\log\pi_{\parameter}=\nabla_{\parameter}^2\phi_{\parameter}-\ex{\pi_{\parameter}}{\nabla_{\parameter}^2\phi_{\parameter}}-\ex{\pi_{\parameter}}{\nabla_{\parameter}\phi_{\parameter}\nabla_{\parameter}\phi_{\parameter}^{\T}}+\ex{\pi_{\parameter}}{\nabla_{\parameter}\phi_{\parameter}}\ex{\pi_{\parameter}}{\nabla_{\parameter}\phi_{\parameter}^{\T}}$.
    Then, we have
        \begin{equation*}
    \begin{aligned}
    \norm{H(\parameter)}
    \leq &\Ebb_{\pi_{\parameter}}\lrb{\norm{\nabla_{\parameter}f_{\parameter} \nabla_{\parameter}\phi_{\parameter}^{\T}}}+2\Ebb_{\pi_{\parameter}}\lrb{|f_{\parameter}|^2}^{1/2}\Ebb_{\pi_{\parameter}}\lrb{\norm{\nabla^{2}_{\parameter}\phi_{\parameter}}^2}^{1/2}+6 M B^2\\
    \leq &  B L_1+2M L_2+6 M B^2,
    \end{aligned}
    \end{equation*}
    where the last inequality is due to Assumptions \ref{asm:local} and \ref{asm:wavefun}.
    Let $L = B L_1+2M L_2+6 M B^2$, then $\Lcal(\parameter)$ is $L$-smooth. 
\end{proof}

According to the book \cite{wright1999numerical}, the $L$-smoothness \eqref{eq:gLip} is also equivalent to 
\begin{equation}
    \label{eq:gLip2}
    \Lcal(\parameter_2)\leq \Lcal(\parameter_1)+\langle g(\parameter_1),\parameter_2-\parameter_1\rangle+\frac{L}{2}\|\parameter_1-\parameter_2\|^2.
\end{equation}

It is already known that the stochastic gradient $\hat{g}$ is biased, unlike the unbiased estimator in the classical SGD. However, the bias decays with increasing burn-in time $n_0$ or the sample size $n$,  which can be regarded sufficiently small. We discuss the expected descent of the function value within each iteration in the following lemma. 

\begin{lemma}
    \label{lem:decrease}
    If the stepsize $\alpha_k\leq \frac{1}{2L}$, for any given $\parameter_k$ the objective function $\Lcal(\parameter_{k+1}) $ decreases in expectation as
    \begin{equation}
        \label{eq:descent}
        \Lcal(\parameter_{k})-\epct{\Lcal(\parameter_{k+1})|\parameter_k}\geq \frac{\alpha_k}{2}\|\nabla\Lcal(\parameter_{k})\|^2-\frac{\alpha_k}{2}\cdot B_{n,n_0}^2
        -\frac{\alpha_k^2L}{2} \cdot V_{n,n_0}.
    \end{equation}
    where  $B_{n,n_0}$ and $V_{n,n_0}$ are defined in Theorem \ref{thm:gbiasVar}.
\end{lemma}
\begin{proof}
    For convenience, we simplify the notations with $\Lcal_k:= \Lcal(\parameter_{k})$, $g_k:=g(\parameter_k)$ and $\hat{g}_k:=\hat{g}(\parameter_k,S_k)$ for any $k\geq 1$.
    By the Lemma \ref{prop:gLip} and the property \eqref{eq:gLip2}, we perform a gradient descent analysis of the iteration \eqref{eq:iter}, 
    \begin{equation}\label{eqn:sgd-s1}
        \begin{aligned}
            \Lcal_{k+1}\leq & \Lcal_k+\langle g_k,\parameter_{k+1}-\parameter_k\rangle+\frac{L}{2}\|\parameter_{k+1}-\parameter_k\|^2\\
            =&\Lcal_k-\alpha_k\langle g_k,\hat{g}_k-g_k\rangle-\alpha_k\|g_k\|^2+\frac{\alpha_k^2 L}{2}\|\hat{g}_k\|^2\\
            =&\Lcal_k-(\alpha_k-\alpha_k^2L)\langle g_k,\hat{g}_k-g_k\rangle-\left(\alpha_k-\frac{\alpha_k^2L}{2}\right)\|g_k\|^2+\frac{\alpha_k^2 L}{2}\|\hat{g}_k-g_k\|^2.
        \end{aligned}
    \end{equation}
Taking the conditional expectation of \eqref{eqn:sgd-s1} on $\parameter_k$ gives
\begin{equation}\label{eqn:sgd-s2}
    \begin{aligned}
        \cond{\Lcal_{k+1}}{\parameter_{k}}\leq
        &\Lcal_k-\left(\alpha_k-\frac{\alpha_k^2L}{2}\right)\|g_k\|^2\\
        &-(\alpha_k-\alpha_k^2L)\langle g_k,\cond{\hat{g}_k}{\parameter_{k}}-g_k\rangle+\frac{\alpha_k^2 L}{2}\cond{\nrm{\hat{g}_k-g_k}^2}{\parameter_k}.
    \end{aligned}
\end{equation}
Through the fact $|\langle x,y\rangle|\leq (\nrm{x}^2+\nrm{y}^2)/2$, it holds
\begin{equation}\label{eqn:sgd-s3}
\begin{aligned}
    -(\alpha_k-\alpha_k^2L)\iprod{g_k}{\cond{\hat{g}_k}{\parameter_{k}}-g_k}
    \leq& \frac{\alpha_k-\alpha_k^2L}{2}\nrm{g_k}^2
    +\frac{\alpha_k}{2}\nrm{\cond{\hat{g}_k}{\parameter_{k}}-g_k}^2.
\end{aligned}
\end{equation}
Therefore, we can plug in \eqref{eqn:sgd-s3} to \eqref{eqn:sgd-s2} and rearrange to get
\begin{equation}
    \label{eq:descent2}
    2\left(\Lcal_k-\cond{\Lcal_{k+1}}{\parameter_{k}}
    \right)
    \geq \alpha_k\nrm{g_k}^2
    -\alpha_k \nrm{\cond{\hat{g}_k}{\parameter_{k}}-g_k}^2-L\alpha_k^2 \cond{\nrm{\hat{g}_k-g_k}^2}{\parameter_k}.
\end{equation}
Thus, \eqref{eq:descent} holds when we substitute the error bounds in Theorem \ref{thm:gbiasVar} into \eqref{eq:descent2}.
\end{proof}

As $n$ goes towards positive infinity, the objective function decreases as an exact gradient descent. However, $n$ and $n_0$ are not too large because of the high sampling cost in practice. The error bounds studied in Section \ref{subsec:error} are valuable for our analysis of the stochastic optimization. Since $\Lcal(\parameter)$ is non-convex, we consider the convergence rate in terms of the expected norm of the gradient $\Ebb\norm{\nabla_{\parameter}\Lcal(\parameter)}^2$. Finally, we establish our first-order convergence results as follows.  
\begin{theorem}
    \label{thm:expt}
    Let Assumptions \ref{asm:local}, \ref{asm:local}, \ref{asm:wavefun} and \ref{asm:unigap} hold and $\{\parameter_k\}$ be generated by Algorithm \ref{alg:VMC}. If the stepsize satisfies $\alpha_k\leq \frac{1}{2L}$, 
        then for any $K$, we have
        \begin{equation}
            \label{eq:converge1}
            \min_{1\leq k\leq K}\Ebb \|g(\parameter_k)\|^2 \leq O\left(\frac{1}{\sum_{k=1}^{K}\alpha_k}\right)+O\left(\frac{\sum_{k=1}^{K}\alpha_k^2}{n\sum_{k=1}^{K}\alpha_k}\right)+O\left(\frac{(1-\gamma)^{2n_0}}{n^2}\right),
        \end{equation}
        where $O(\cdot)$ hides constants $\gamma,M,B,C$ and retains $n,n_0$ and $K$. In particular, if the stepsize is chosen as  $\alpha_k=\frac{c\sqrt{n}}{\sqrt{k}}$ where $c\leq\frac{1}{2L} $, then we have
        \begin{equation}
            \label{eq:converge2}
            \min_{1\leq k\leq K}\Ebb \|g(\parameter_k)\|^2 \leq O\left(\frac{\log K}{\sqrt{nK}}\right)+O\left(\frac{(1-\gamma)^{2n_0}}{n^2}\right).
        \end{equation}
\end{theorem}
\begin{proof}

Lemma \ref{lem:decrease} suggests that, for any $k\geq 1$, 
\begin{equation*}
    \begin{aligned}
        \alpha_k\nrm{g(\parameter_k)}^2
        \leq & 2\left(\Lcal(\parameter_k)-\cond{\Lcal(\parameter_{k+1})}{\parameter_{k}}\right)+\alpha_k\cdot B_{n,n_0}^2+\alpha_k^2\cdot L V_{n,n_0}.%\\
        %B_{n,n_0}\leq & O\left(\frac{\rchi^{\frac{1}{2}}\gamma^{\frac{n_0}{2}}\sigma}{\sqrt{n}}\right)+O\left(\frac{\rchi\gamma^{n_0}}{n}\right) ,\\
        %V_{n,n_0}\leq & O\left(\frac{(1+\rchi\gamma^{n_0})\sigma^2}{n}\right)+O\left(\frac{1}{n^2}\right).
    \end{aligned}
\end{equation*}
Thus, taking total expectation and summing over $k=1,\cdots,K$ yields
\begin{align*}
    &\left(\sum_{k=1}^{K}\alpha_k \right)\min_{1\leq k\leq K} \epct{\nrm{g(\parameter_k)}^2} \leq~ \sum_{k=1}^{K}\alpha_k \epct{\nrm{g(\parameter_k)}^2}
    \\
    &\leq~ 2\left(\Lcal(\parameter_1)-\epct{\Lcal(\parameter_{K+1})}
    \right)
    +\sum_{k=1}^{K}\alpha_k^2 \cdot LV_{n,n_0}
    + \sum_{k=1}^{K}\alpha_k\cdot B_{n,n_0}^2\\
    &=~O(1)+\sum_{k=1}^{K}\alpha_k^2 \cdot O\left(\frac{1}{n}\right)+\sum_{k=1}^{K}\alpha_k \cdot O\left(\frac{(1-\gamma)^{2n_0}  }{n^2}\right).
\end{align*}
Divide both sides by $\sum_{k=1}^{K}\alpha_k$, then \eqref{eq:converge1} holds. If we take $\alpha_k=\frac{c\sqrt{n}}{\sqrt{k}}$, \eqref{eq:converge2} is implied by
\begin{equation*}
        \sum_{k=1}^{K}\alpha_k=\sum_{k=1}^{K}\frac{c\sqrt{n}}{\sqrt{k}}=O(\sqrt{nK}),\quad \sum_{k=1}^{K}\alpha_k^2=\sum_{k=1}^{K}\frac{c^2n}{k}=O(n\log K).
\end{equation*}
This completes the proof.
\end{proof}

Theorem \ref{thm:expt} shows the convergence rate of the MCMC-SGD with certain choices of stepsizes. The convergence rate is related to the bias and variance produced by the MCMC estimators. A trade-off involves balancing the convergence rate of the algorithm with the sample size. As the bias term depends on $n^2$, we can select an appropriate sample size $n$ to ensure that the bias term has few impact on convergence. Also, in practical applications, we often choose a long burn-in period at the start of the algorithm to reduce the bias in sampling.

% escape from saddle point

\section{Escaping from saddle points}
\label{sec:saddle}
In this section, we discuss how MCMC-SGD escapes from saddle points and converges to second-order stationarity. Daneshmand et al. first propose the CNC assumption, by which the simple SGD without explicit improvement is proven to have second-order convergence \cite{daneshmand2018escaping}. Inspired by the ideas of this study, we proceed to analyze the escape from saddle points under biased gradient estimators. 
To begin with, a second moment lower bound for non-stationary Markov chains is developed to guarantee the sufficient noise of the biased gradient. Then, we verify that the CNC condition under errors is satisfied when our assumptions hold. Eventually, we demonstrate our convergence analysis of Algorithm \ref{alg:VMC} to reach approximate second-order stationary points with high probability.

% We first give the definition of kurtosis. Let $\sigma_2^{2}[X]:=\Ebb[\left(X-\Ebb[X]\right)^2]$ be the variance and $\sigma_4^4[X]:=\Ebb[\left(X-\Ebb[X]\right)^4]$ be the fourth central moment for any random variable $X$. The kurtosis of $X$ is the fourth standardized moment, defined as
% \begin{equation*}
%     \kappa[X]:=\frac{\sigma_4^4[X]}{\sigma_2^{4}[X]}=\frac{\Ebb[\left(X-\Ebb[X]\right)^4]}{\big(\Ebb[\left(X-\Ebb[X]\right)^2]\big)^2}.
% \end{equation*}
% Kurtosis means the "tailedness" of the probability distribution of a random variable. This number is related to the tails of the distribution.

To deal with the second-order structure, the following assumption is needed. 
\begin{assumption}
    \label{asm:sec} 
    (1) The Hessian matrix $H(\parameter)$ defined in \eqref{eq:Hess} is $\rho$-Lipschitz continuous, i.e.,
    \begin{equation*}
       \norm{H(\parameter_1)-H(\parameter_2)}\leq \rho\norm{\parameter_1-\parameter_2}, ~\forall \parameter_1,\parameter_2\in\Theta.
   \end{equation*}
    
   (2) Let $v$ be the unit eigenvector with respect to the minimum eigenvalue of the Hessian matrix $H(\parameter)$.  There exists a constant $\eta>0$ such that 
   \begin{equation*}
    \label{eq:glowerbound}
    \Ebb_{\pi_{\parameter}}\left[\left(f_{\parameter}(x)v^{T}\left(\nabla_{\parameter}\phi_{\parameter}(x)-\ex{\pi_{\parameter}}{\nabla_{\parameter}\phi_{\parameter}(x)}\right)\right)^2 \right]\geq \eta \sigma_{2}^2[f_{\parameter}],~\forall \parameter \in \Theta.
   \end{equation*}

   (3) For any $\parameter\in \Theta$, the ratio of the sub-exponential norm of $f_{\parameter}$ to its variance has an upper bound $\kappa>0$, that is,
   \begin{equation*}
    \sup_{\parameter\in \Theta}\left(\norm{f_{\parameter}}_{\psi_1(\pi_{\parameter},\Xcal)}/\sigma_2[f_{\parameter}]\right)\leq \kappa.
   \end{equation*}
\end{assumption}

The first assumption is common to perform a second-order convergence analysis. The second assumption refers to the Correlated negative curvature (CNC) condition introduced by Daneshmand et al. \cite{daneshmand2018escaping}, which assumes that the second moment of the projection of stochastic gradients along the negative curvature direction is uniformly bounded away from zero. By this way, the simple SGD without explicit improvement is proven to have second-order convergence. However, the CNC condition is difficult to satisfy when the stochastic gradient is almost surely zero. We modify their CNC condition into the more precise one by introducing the variance of the function $\sigma_2^2(f_{\theta})$. 
The third one guarantees that the $f_{\parameter}$ will have a respectively light-tailed distribution. It holds for most of distributions and especially near the ground state of quantum many-body problems.

Besides, over Assumptions \ref{asm:local} and \ref{asm:wavefun}, we can assume an upper bound of the stochastic gradient to simply our analysis:
\begin{equation}
    \label{eq:gradientbound}
    \begin{aligned}
        l_g:=\sup_{\parameter\in \Theta}\sup_{x\in \Xcal}\norm{\left(f_{\parameter}(x)-\ex{x\sim \pi_{\parameter}}{f_{\parameter}(x)}\right)\nabla_{\parameter}\phi_{\parameter}(x)}<+\infty.
        %\Ebb\norm{\hat{g}(\parameter;S)}^2\leq 2\norm{g(\parameter)}^2+2\Ebb\norm{\hat{g}(\parameter;S)-g(\parameter)}^2\leq 2l^2+2V_{n,n_0}=:l_g^2,
    \end{aligned}
\end{equation}
This treatment is legal because $f_{\parameter}$ is assumed to sub-exponential. Otherwise, the proof can be also established with similar techniques in section \ref{subsec:error}. The upper bound will simplify our proof in Lemma \ref{lem:R2}.

\subsection{The correlated negative curvature condition under errors}
\label{subsec:cnc}
We first prove a lemma about the second moment for non-stationary Markov chains. This lemma establishes that the second moment of the MCMC estimator is bounded from below by a positive quantity relying on the sample size $n$, the spectral gap $\gamma$ and $\Ebb_{\pi}[h^2]$. 
\begin{lemma}
    \label{lem:lowerbound}
    Let $\{X_i\}_{i= 1}^{n_0+n}$ be the Markov chain with the stationary distribution $\pi$ and the initial distribution $\nu$. Suppose it admits a absolute spectral gap $\gamma$ and $\rchi=\rchi(\nu,\pi)<+\infty$. The function $h$ has a finite fourth central moment $\sigma_4^4[h]:=\Ebb_{\pi}[\left(h-\Ebb_{\pi}[h]\right)^4]$. If $n\geq \frac{32}{\gamma^3}$ and $n_0\geq\frac{2}{\gamma}(\log \rchi+\log(\sigma_4[h]/\sigma_2[h])+\log n)$, it holds that  
\begin{equation}
    \label{eq:lowerbound}
    \Ebb_{\nu}\left[\left(\frac{1}{n}\sum_{i=n_0+1}^{n+n_0}h(X_i)\right)^2\right]\geq \frac{\gamma}{4n}\Ebb_{\pi}[h^2].
\end{equation}

\end{lemma}
\begin{proof}
    We denote $Z=\frac{1}{n}\sum_{i=n_0+1}^{n+n_0}h(X_i)$ and $c=\Ebb_\pi[h]$. Then it holds
    \begin{equation}
        \label{eq:lb1}
        \Ebb_\nu[Z^2]=c^2+\Ebb_\nu[(Z-c)^2]+2c(\Ebb_\nu[Z]-c)
        \geq \frac{c^2}{2}+\Ebb_\nu[(Z-c)^2]-2(\Ebb_\nu[Z]-c)^2.
    \end{equation}
    Notice the difference 
    \begin{equation}
        \label{eq:diff}
        \begin{aligned}
            \left| \Ebb_{\nu}\left[(Z-c)^2\right] - \Ebb_{\pi}\left[(Z-c)^2\right]\right|
             \leq (1-\gamma)^{n_0}\rchi \cdot \left(  \Ebb_{\pi}\left[(Z-c)^4\right] \right)^{\frac{1}{2}}
        \end{aligned}
    \end{equation}
    where the inequality changes the measure as in \eqref{eq:change-measure}. It follows from the Cauchy-Schwarz inequality that
    \begin{align*}
        \Ebb_{\pi}\left[(Z-c)^4\right]
        =&~\Ebb_{\pi}\left[\left(\frac{1}{n}\sum_{i=n_0+1}^{n+n_0} h(X_i)-\Ebb_\pi[h]\right)^4\right] \\
        \leq&~ \Ebb_{\pi}\left[\frac{1}{n}\sum_{i=n_0+1}^{n+n_0} \left(h(X_i)-\Ebb_\pi[h]\right)^4\right] = \Ebb_{X\sim \pi}\left[\left(f(X)-\Ebb_\pi[h]\right)^4\right] = \sigma_4^4[h],
    \end{align*}
    where the second equality is because $\pi$ is the stationary distribution of the Markov chain.
    Therefore, we derive that
    \begin{equation}
        \label{eq:lb2}
        \Ebb_{\nu}\left[(Z-c)^2\right] \geq \Ebb_{\pi}\left[\left(\frac{1}{n}\sum_{i=n_0+1}^{n+n_0} h(X_i)-\Ebb_\pi[h]\right)^2\right] - (1-\gamma)^{n_0}\rchi\sigma_4^2[h].
    \end{equation}
    It holds from Theorem 3.1 in \cite{paulin2015concentration} that
    \begin{equation}
        \label{eq:lb3}
        \begin{aligned}
            \left|\Ebb_{\pi}\left[\left(\frac{1}{n}\sum_{i=n_0+1}^{n+n_0} h(X_i)-\Ebb_\pi[h]\right)^2 \right]-\frac{\sigma^2_{asy}[h]}{n}\right|\leq \frac{16 \sigma^2_2[h]}{\gamma^2n^2},
        \end{aligned}
    \end{equation}
    where $\sigma^2_{asy}[h]:=\langle h,[2(I-(\oP-\pi))^{-1}-I]h\rangle_{\pi}$ on page 24 of \cite{paulin2015concentration}. Using the spectral method therein, we obtain that 
    \begin{equation}
        \label{eq:lb4}
        \begin{aligned}
        \sigma^2_{asy}[h]&=\langle h,[2(I-(\oP-\pi))^{-1}-I]h\rangle_{\pi}=\norm{(I-\oP)^{-1}h}_{\pi}^2-\norm{P(I-P)^{-1}h}_{\pi}^2\\
        &\geq (1-(1-\gamma)^2)\norm{(I-\oP)^{-1}h}_{\pi}^2\geq \frac{\gamma\sigma^2_2[h]}{2}.
    \end{aligned}
    \end{equation}
    Next, we bound the term $(\Ebb_\nu[Z]-c)^2$ in \eqref{eq:lb1}. By Lemma \ref{lem:BernBiasVar}, it holds that
    \begin{equation}
        \label{eq:lb5}
        \abs{ \Ebb_\nu[Z]-c } = \abs{ \frac{1}{n}\sum_{i=n_0+1}^{n+n_0} \Ebb_{\pi} \left[h(X_i)\right]-\Ebb_\pi[h] } \leq \frac{(1-\gamma)^{n_0} \rchi \sigma_2[h] }{ \gamma n}.
    \end{equation}
    Finally, combining \eqref{eq:lb1}, \eqref{eq:lb2}, \eqref{eq:lb3}, \eqref{eq:lb4} and \eqref{eq:lb5} yields
    \begin{align*}
        & ~\Ebb_{\nu}\left[\left(\frac{1}{n}\sum_{i=n_0+1}^{n+n_0}h(X_i)\right)^2\right]\\
        &\geq \frac{1}{2}\left(\Ebb_\pi[h]\right)^2 + \frac{\gamma\sigma^2_2[h]}{2n}-\frac{16\sigma^2_2[h]}{\gamma^2n^2}- (1-\gamma)^{n_0}\rchi\sigma_4^2[h] -\frac{2(1-\gamma)^{2n_0}\rchi^2\sigma^2_2[h]}{\gamma^2 n^2}\\
        &\geq \sigma^2_2[h]\left(\frac{\gamma}{n}-\frac{16}{\gamma^2n^2}-\frac{2(1-\gamma)^{2n_0}\rchi^2}{\gamma^2 n^2}-(1-\gamma)^{n_0}\rchi\cdot \frac{\sigma^2_4[h]}{\sigma^2_2[h]}\right)+ \frac{1}{2}\left(\Ebb_\pi[h]\right)^2
        \\
        &\geq \frac{\gamma}{4n}\sigma^2_2[h] + \frac{1}{2}\left(\Ebb_\pi[h]\right)^2
        \geq \frac{\gamma}{4n}\Ebb_{\pi}[h^2],
    \end{align*}
    where the third inequality holds when $n\geq \frac{32}{\gamma^3}$ and $n_0\geq\frac{2}{\gamma}(\log \rchi+\log(\sigma_4[h]/\sigma_2[h]) + \log n)$.
\end{proof}

The CNC condition under errors \eqref{eq:cnc} in the following is of vital importance to analyze the saddle escaping property in SGD. That is one of the reason why stochastic optimization algorithms can escape from saddle points. If the CNC condition does not hold, the noise may not be strong enough to escape along the descent direction. The following lemma shows that the empirical stochastic gradient $\hat{g}(\parameter;S)$ defined in \eqref{eq:approxgrad} satisfies the CNC condition under errors.
\begin{lemma}
    \label{lem:cnc}
    Let Assumptions \ref{asm:wavefun},\ref{asm:local}, \ref{asm:unigap} and  \ref{asm:sec} hold. Then, for the unit eigenvector $v$ with respect to the minimum eigenvalue of the Hessian, there exists $\mu=\frac{\eta \gamma}{16n}$ such that it holds 
    \begin{equation}
        \label{eq:cnc}
        \Ebb_{\nu}\left[\left(v^{\T}\hat{g}(\parameter;S)\right)^2\right]\geq \mu\sigma_2^2[f_{\parameter}],\quad \forall \parameter\in \Theta,
    \end{equation}
    as long as $n\geq \Omega\left(\frac{1}{\eta\gamma^3}\right)+\tilde{\Omega}\left( \frac{\kappa B }{\sqrt{\eta}\gamma^2}\right)$  and $n_0\geq\frac{2}{\gamma}(\log \rchi+\log(2\kappa) + \log n)$ where $\Omega(\cdot)$ and $\tilde{\Omega}(\cdot)$ hide constants $M,B,C,\rchi$. 
\end{lemma}
\begin{proof}
   The MCMC estimate of the gradient $\hat{g}(\parameter,S)$ is computed by \eqref{eq:approxgrad}. For convenience, fix a unit vector $v$ and we denote empirical variables by $F_i=f_{\parameter}(x_i),Y_i=v^{\T}\nabla_{\parameter}\phi_{\parameter}(x_i)$ and $\overline{F},\overline{Y}$ represents the average of $E_i,Y_i$. Let stationary variables $F=f_{\parameter}(x)$ and $Y=v^{\T}\nabla_{\parameter}\phi_{\parameter}(X)$ with a dependent variable $X\sim \pi_{\parameter}$, then it holds that
    \begin{align*}
        v^{\T}\hat{g}(\parameter,S)&=\frac{2}{n}\sum_{i=n_0+1}^{n_0+n}\big(f_{\parameter}(x_i)-\frac{1}{n}\sum_{i=n_0+1}^{n_0+n}f_{\parameter}(x_i)\big)v^{\T}\nabla \phi_{\parameter}(x_i)\\
        &=\frac{2}{n}\sum_{i=n_0+1}^{n_0+n}F_iY_i-2\left(\overline{F}-\Ebb F+\Ebb F\right)\left(\overline{Y}-\Ebb Y+\Ebb Y\right)\\
        &=\underbrace{\frac{2}{n}\sum_{i=n_0+1}^{n_0+n}(F_i-\Ebb F)(Y_i-\Ebb Y) }_{Z_1}-\underbrace{2\left(\overline{F}-\Ebb F\right)\left(\overline{Y}-\Ebb Y\right)}_{Z_2}.
    \end{align*}
To obtain the lower bound, we have
\begin{align*}
    \Ebb_{\nu}\left[(v^{\T}\hat{g}(\parameter,S))^2\right]=\Ebb_{\nu} \left[(Z_1-Z_2)^2\right]\geq \frac{1}{2}\Ebb_{\nu} \left[Z_1^2\right]-\Ebb_{\nu} \left[Z_2^2\right].
\end{align*}
Since $e^{x}> x^4/8$ when $x\geq 0$, it holds that $\sigma_{4}[F]\leq 2 \norm{F}_{\psi_1(\pi_{\parameter},\Xcal)}$. It implies the kurtosis $\sigma_4[F]/\sigma_2[F]$ is less than $2\kappa$.
Then, when $n\geq \frac{32}{\gamma^3}$ and $n_0\geq\frac{2}{\gamma}(\log \rchi+\log(2\kappa) + \log n)$, it follows from Lemma \ref{lem:lowerbound} that 
\begin{equation}
    \label{eq:Z1}
    \begin{aligned}
        \Ebb_{\nu}\left[ Z_1^2\right]& \geq \frac{\gamma}{4n}\Ebb_{\pi}[(F-\Ebb F)^2 (Y-\Ebb Y)^2]{\geq}\frac{\eta\gamma \sigma_2^2[F]}{4n}.
    \end{aligned}
\end{equation}

Besides, the upper bound of $ \Ebb_{\nu} \left[Z_2^2\right]$ can be obtained by the Cauchy-Schwarz inequality that
\begin{equation}
    \label{eq:z2}
    \begin{aligned}
        \Ebb_{\nu} [Z_2^2]&=4\Ebb_{\nu}[(\bar{F}-\Ebb F)^2(\bar{Y}-\Ebb Y)^2]\leq 4\sqrt{\Ebb_{\nu}[(\bar{F}-\Ebb F)^4]\cdot \Ebb_{\nu}[(\bar{Y}-\Ebb Y)^4]}.
    \end{aligned}
\end{equation}
By the result in the proof of Lemma \ref{lem:BernBiasVar}, we have
\begin{equation}
    \label{eq:Xsquare}
    \Ebb_\nu[\abs{\bar{E}-\Ebb F}^4]\leq \bigO{ \frac{\sigma_2^4[E_{\parameter}]}{n^2\gamma^2}+\frac{M^4(\log n)^8}{n^{4}\gamma^4} }.
\end{equation}
% To estimate fourth central moment, we use the high probability bound in \eqref{eq:highprobbound}. Reusing notations in Lemma \ref{lem:BernBiasVar}, we write that for any $\delta>0$, it holds that
% \begin{equation*}
%     \Pbb\left(\abs{\Delta}\leq s\right)\geq 1-\delta,\quad s\geq\max\left\{\frac{8\sigma_2[f]\sqrt{\log(4C/\delta)}}{\sqrt{n \gamma}},\frac{40M\left(\log(4Cn/\delta)\right)^2 }{n\gamma}\right\}.
% \end{equation*}
% By computing the following integral, we obtain the fourth moment bound
% \begin{equation}
%     \label{eq:fourthmoment}
%     \Ebb_\nu[\abs{\Delta}^4]=\int_{0}^{+\infty}4 s^3 \Pbb(\abs{\Delta}>s)ds\leq \frac{8^5C\sigma_2^4[f]}{n^2\gamma^2}+\frac{10080 C(80M)^4}{n^{3}\gamma^4}.
% \end{equation}
% Under Assumption \ref{asm:sec}, it follows from \eqref{eq:fourthmoment} and $n\geq \frac{4096}{\gamma^3}$ that
% \begin{equation}
%     \label{eq:Xsquare}
%     \Ebb(\bar{X}-\Ebb X)^2\leq \frac{8^5C\sigma_2^4[E_{\parameter}]}{n^2\gamma^2}+\frac{10080 C(80M)^4}{n^{3}\gamma^4}\leq \frac{40^5\kappa^4C\sigma_2^4[E_{\parameter}]}{n^2\gamma^2}.
% \end{equation}
Similarly, notice that $\abs{Y}=\abs{v^{\T}\nabla_{\parameter}\phi_{\parameter}(x)}\leq \norm{v}\norm{\nabla_{\parameter}\phi_{\parameter}(x)}\leq B$, we can apply \cite[Theorem 12]{fan2021hoeffding} to obtain 
\begin{equation*}
    \Pbb\left(|\bar{Y}-\Ebb Y|\geq t \right)\leq 2C \exp\left(-\frac{n\gamma t^2}{4B^2} \right).
\end{equation*}  
Using the similar technique in the proof of Lemma \ref{lem:Bern-exp}, we derive
\begin{equation}
    \label{eq:Ysquare}
    \Ebb_{\nu}[|\bar{Y}-\Ebb Y|^4]\leq \bigO{ \frac{B^4}{n^2\gamma^2} }.
\end{equation}
% \begin{equation*}
%     \Pbb\left(|\bar{Y}-\Ebb Y|\geq t \right)\leq 2C \exp\left(-\frac{n\gamma t^2}{4B^2} \right).
% \end{equation*}  
% It implies that
% \begin{equation}
%     \label{eq:Ysquare}
%     \Ebb_{\nu}[|\bar{Y}-\Ebb Y|^4]=\int_{0}^{+\infty}4t^3\Pbb\left(|\bar{Y}-\Ebb Y|\geq t \right)dt\leq \frac{64C B^4}{n^2\gamma^2}.
% \end{equation}
Plugging \eqref{eq:Xsquare} and \eqref{eq:Ysquare} into \eqref{eq:z2} yields
\begin{equation*}
    \Ebb_{\nu} [Z_2^2]\leq 4\sqrt{\Ebb_{\nu}[(\bar{F}-\Ebb F)^4]\cdot \Ebb_{\nu}[(\bar{Y}-\Ebb Y)^4]}
    \leq \bigO{ \frac{\sigma_2^{2}[F]B^2}{n^2\gamma^2} + \frac{(\log n)^4M^2B^2}{n^3\gamma^3} }.
    % \leq \frac{409600\kappa^2CB^2\sigma_2^{2}[E_{\parameter}]}{n^2\gamma^2}\leq \frac{\eta\gamma \sigma_2^2[E_{\parameter}]}{32n},
\end{equation*}
Therefore, $\Ebb_{\nu} [Z_2^2]\leq \frac{\eta\gamma \sigma_2^2[F]}{32n}$ as long as $n\geq \Omega\left(\frac{1}{\eta\gamma^3}\right)+\tilde{\Omega}\left( \frac{\kappa B }{\sqrt{\eta}\gamma^2}\right)$.
% when $n\geq \frac{c\kappa^2CB^2}{\eta\gamma^3}$ with a constant $c$. 
Finally, combining our discussion of $Z_1$ and $Z_2$, we derive that
\begin{equation*}
    \Ebb_{\nu}\left[(v^{\T}\hat{g}(\parameter,S))^2\right]\geq \frac{1}{2}\Ebb_{\nu} \left[Z_1^2\right]-\Ebb_{\nu} \left[Z_2^2\right]\geq \frac{\eta\gamma \sigma_2^2[F]}{32n}.
\end{equation*}
This completes the proof.
\end{proof}

\subsection{Second order stationarity}

As discussed in Section \ref{sec:conv}, MCMC-SGD converges to first order stationary points. While the convergence of SGD is well-understood for convex functions, the existence of saddle points and local minimum poses challenges for non-convex optimization. Since the ordinary GD often stucks near the saddle points, the additional noise within SGD allows it to escape from saddle points. We have verified the CNC condition for biased gradients, by which we are able to analyze the behavior of MCMC-SGD near saddle points. Compared to previous studies \cite{daneshmand2018escaping,ge2015escaping,jin2017escape}, the presence of bias in the stochastic gradient needs a more nuanced analysis.

We first give the definition of approximate second-order stationary points.
\begin{definition}[Approximate second-order stationary point]
    Given a function $\Lcal$, an $(\epsilon_g,\epsilon_h)$ approximate second-order stationary point $\parameter$ of $\Lcal$ is defined as 
    \begin{equation*}
        \norm{g(\parameter)}\leq \epsilon_g,\quad \lambda_{\min}\left(H(\parameter)\right)\geq -\epsilon_h,
    \end{equation*}
    where $g$ and $H$ denote the gradient and Hessian of $\Lcal$ respectively.
\end{definition}
    
If $\epsilon_g=\epsilon_h=0$, the point $\parameter$ is a second-order stationary point. The second order analysis contributes to understand how MCMC-SGD leaves from these saddle points and finally reach approximate second-order stationary points. 

An observation lies that when the variance of the function $f_{\parameter}(x)$ approaches zero, the efficacy of MCMC-SGD is compromised, as the stochastic gradient tends to zero. Notably, in the context of variational eigenvalue problems, the specific situation of interest arises when an eigenvalue is obtained as $\sigma_2^2[f_{\theta}]=0$. This scenario is precisely what is required for our analysis. Hence, we define $\epsilon$-variance points as another criteria for the algorithm \ref{alg:VMC}. 
\begin{definition}[$\epsilon$-variance point]
    For the optimization problem \eqref{eq:loss}, we call $\parameter$ an $\epsilon$-variance point if the function $f_{\parameter}$ satisfies $\sigma_2^2[f_{\parameter}]< \epsilon$.
\end{definition}
This definition may differentiate the regions where modeling and parameterization cause the algorithm to get stuck. For notational simplification, we take an abbreviation $\sigma_2^2=\sigma_2^2[f_{\parameter}]$ in this subsection.

To escape from saddle points, a new stepsize schedule is adopted for Algorithm \ref{alg:VMC}. Given a period $T$, note that $\alpha$ and $\beta$ are the constant stepsizes with $\beta>\alpha>0$, with values given in Table \ref{table:1}. Within one iteration period of $T$ steps, we adopt a large stepsize $\beta$ at the beginning of the period and a small one $\alpha$ at the other $T-1$ iterations, that is,
\begin{equation}
    \label{eq:schedule}
    \alpha_k=\begin{cases}
        \alpha, \quad k (\mathrm{mod}T)\not =0,\\
        \beta,  \quad k (\mathrm{mod}T) =0.
    \end{cases}
\end{equation}
It will be shown that the schedule can be suitably designed to achieve sufficient descent in one period.

Suppose the total number of iterations $K$ is a multiple of $T$ and there are $K/T$ periods. We denote $\tparameter_m=\parameter_{m\cdot T}$ for $m=0,1,\dots,K/T$ in each period. For some given $\epsilon>0$, we consider four regimes of the iterates $\{\tparameter_m\}$ as follow,
\begin{align}
    \Rcal_1&:=\left\{\parameter\Big| \norm{g(\parameter)}\geq \epsilon\right\},\label{eq:R1}\\
    \Rcal_2&:=\left\{\parameter\Big| \norm{g(\parameter)}< \epsilon,~\lambda_{\min}\left(H(\parameter)\right)\leq -\epsilon^{\frac{1}{4}} ,~~\sigma_2^2[f_{\parameter}]\geq \epsilon^{\frac{1}{2}} \right\},\label{eq:R2}\\
    \Rcal_3&:=\left\{\parameter\Big| \norm{g(\parameter)}< \epsilon,~\lambda_{\min}\left(H(\parameter)\right)>-\epsilon^{\frac{1}{4}} ,~\mathrm{or}~\sigma_2^2[f_{\parameter}]<\epsilon^{\frac{1}{2}} \right\},\label{eq:R3}
\end{align}
$\Rcal_1$ stands for the regime with a large gradient, where the stochastic gradient works effectively. When the iterate lies in $\Rcal_2$, despite being close to a first-order stationary point, the CNC condition mentioned in section \ref{subsec:cnc} guarantees a decrease after $T$ iterations under our schedule. $\Rcal_3$ is a regime of $(\epsilon,\epsilon^{1/4})$ approximate second order stationary points or $\epsilon^{1/2}$ variance points. We need to show Algorithm \ref{alg:VMC} will reach $\Rcal_3$ with high probability, that is, converge to approximate second order stationary points.

The analysis below relies on a particular choice of parameters, whose values satisfy the following lemmas and the main theorem. For ease of verification, the choice of parameters is collected in Table \ref{table:1}. 
\begin{table}[ht]
    \centering
    \renewcommand{\arraystretch}{1.8}
    \begin{tabular}{|c|c|c|c|c|c|}
    \hline
     Parameter& Value & Order & Conditions  \\ \hline
     $\beta $ & $\frac{ \delta\epsilon^2}{192 l_g\rho L V_{n,n_0}}$ & $O(\epsilon)$ &     \eqref{eq:R1conditions}, \eqref{eq:conditions_1} \\ \hline
     $\alpha$& $\frac{\beta}{\sqrt{T}}$ & $  O\left(\epsilon^{9/4}\log^{-1}\left(\frac{1}{\epsilon}\right)\right)$&  \eqref{eq:conditions_1}  \\ \hline
     $\Lthre$&$\frac{\beta \epsilon^2 }{192 \rho l_g}$ & $ O(\epsilon^{3})$ &  \eqref{eq:R1conditions},\eqref{eq:conditions_1}  \\ \hline
     $n$&  $\frac{\eta \gamma }{64\epsilon}$ & $O\left(\epsilon^{-1}\right)$ &  Lemma \ref{lem:cnc}  \\ \hline
     $n_0$&  - & $O\left(\log\left(\frac{1}{\epsilon}\right)\right)$ &  \eqref{eq:BV}, Lemma \ref{lem:cnc}  \\ \hline
     $B_{n,n_0}$ & $\sqrt{\frac{\epsilon^{2} }{96 T^{3/2} l_g\rho}} $  & $ O(\epsilon^{3})$ &  \eqref{eq:R1conditions}, \eqref{eq:conditions_1}  \\ \hline
     $T$& $\frac{1}{\beta^2\epsilon^{1/2}}
     \log^2\left(\frac{\rho l_g L V_{n,n_0}}{\mu \delta \epsilon} \right)$ &  $ O\left(\epsilon^{-5/2}\log^2\epsilon\right)$& \eqref{eq:Tconst}    \\ \hline
     $K$& $\frac{2[\Lcal(\parameter_0)-\Lcal^{*}]T}{\delta \Lthre}$ & $ O\left(\epsilon^{-11/2}\log^2\epsilon\right)$ &   \eqref{eq:Kconst}  \\ \hline
    \end{tabular}
    \caption{List of parameters.}
    \label{table:1}
\end{table}

We shall briefly introduce each parameter in Table \ref{table:1} to facilitate the understanding of the following proof. $\beta$ is the larger stepsize that assists SGD in achieving sufficient descent when the gradient norm is large. $\alpha$ is the smaller stepsize that helps SGD escaping saddle points when the algorithm approaches first-order stationary points. $\Lthre$ stands for the sufficient decrease in the expected function value near the large gradient region and saddle points. $n$ and $n_0$ are the sample size and the length of burn-in period of MCMC algorithms. $B_{n,n_0}$ is the bias of MCMC algorithms, depending on $n$ and $n_0$. $T$ represents the period length for switching stepsize $\beta$ and $\alpha$, in which we adopt the large stepsize $\beta$ at the beginning
and the small one $\alpha$ at the other $T-1$ iterations. Finally, $K$ is the total number of iterations of MCMC-SGD. These parameters are chosen such that all the conditions are satisfied. We presuppose the validity of these conditions and then deduce the values of the parameters from them. Although the choice of each parameter may not appear intuitive, it is not necessary to substitute these parameters in their entirety during computation. One only needs to verify the fulfillment of the conditions listed in Table \ref{table:1}.

With a large gradient, it is easy to show a sufficient decrease of the objective function value. We have analyzed the decrease for each iteration in Lemma \ref{lem:decrease}, and it follows a similar argument for the MCMC-SGD.
\begin{lemma}
    \label{lem:R1}
    Suppose that $\tparameter_m$ lies in $\Rcal_1$ defined in \eqref{eq:R1} and Algorithm \ref{alg:VMC} updates with the schedule \eqref{eq:schedule} and parameters in Table \ref{table:1}, then the expected value of $\Lcal(\tparameter_{m+1})$ taken over the randomness of $\{\parameter_{k}\}_{k=m\cdot T+1}^{(m+1)\cdot T}$ decreases as 
    \begin{equation*}
        \Lcal(\tparameter_{m})-\Ebb[\Lcal(\tparameter_{m+1})|\tparameter_{m}]\geq \Lthre.
    \end{equation*}
\end{lemma}
\begin{proof}
    We first decompose the difference of the expected function value into each iteration,
    \begin{equation}
        \label{eq:period-decrease}
        \Lcal(\tparameter_m)-\Ebb[\Lcal(\tparameter_{m+1})|\tparameter_{m}]=\sum_{p=0}^{T -1}\epct{\Lcal(\parameter_{m\cdot T +p})-\epct{\Lcal(\parameter_{m\cdot T +p+1})}\bigg|\parameter_{m\cdot T +p}},
    \end{equation}
    where $\epct{\Lcal(\parameter_{m\cdot T})|\parameter_{m\cdot T }}=\Lcal(\tparameter_{m})$ due to the definition $\tparameter_{m}=\parameter_{m\cdot T}$. Using the choice $\beta^2= T\alpha^2$ in Table \ref{table:1}, it follows from \eqref{eq:period-decrease} that
    \begin{equation}
        \label{eq:R1ineq}
        \begin{aligned}
            2\left(\Lcal(\tparameter_m)-\Ebb[\Lcal(\tparameter_{m+1})|\tparameter_{m}]\right)\geq &\beta\norm{g(\tparameter_m)}^2-\beta B_{n,n_0}^2-\beta^2LV_{n,n_0}\\
            &\quad -(T-1)\left(\alpha B_{n,n_0}^2+\alpha^2LV_{n,n_0} \right)\\
            \geq &\beta\norm{g(\tparameter_m)}^2-2T\alpha B_{n,n_0}^2-2\beta^2LV_{n,n_0},
        \end{aligned}
    \end{equation}
    where the first inequality is by Lemma \ref{lem:decrease} for $p=0,\dots,T-1$ and the second is by the direct substitution.
    Then, by the choice of $\beta,n_0,\Lthre$ in Table \ref{table:1}, these conditions hold that
    \begin{equation}
        \label{eq:R1conditions}
        \beta\leq \frac{\epsilon^2}{8LV_{n,n_0}},\quad B_{n,n_0}^2\leq \frac{\epsilon^2}{8\sqrt{T}},\quad \Lthre \leq \frac{\beta\epsilon^2}{4}.
    \end{equation}
    As $\tparameter_m$ lies in $\Rcal_1$, which means $\norm{\nabla_{\parameter}\Lcal(\tparameter_m)}\geq \epsilon$, we plug \eqref{eq:R1conditions} into \eqref{eq:R1ineq} and obtain that
    \begin{equation*}
        \begin{aligned}
            2\left(\Lcal(\tparameter_m)-\Ebb[\Lcal(\tparameter_{m+1})|\tparameter_{m}]\right)\geq &\beta \epsilon^2-2\beta\sqrt{T}B_{n,n_0}^2-2\beta LV_{n,n_0}\cdot \frac{\epsilon^2}{8LV_{n,n_0}}\\
            \geq& \beta \epsilon^2-\frac{\beta \epsilon^2}{4}-\frac{\beta \epsilon^2}{4}= \frac{\beta\epsilon^2}{2}\geq 2\Lthre.
        \end{aligned}
    \end{equation*}
    This completes the proof.
\end{proof}

Near the saddle points, the classical GD gets stuck if the gradient is orthogonal to the negative curvature direction. However, the stochastic gradient with the CNC condition has inherent noise along the negative curvature direction. Under our stepsize schedule \eqref{eq:schedule}, the objective function value can have sufficient decrease after a period.
\begin{lemma}
    \label{lem:R2}
    Suppose $\tparameter_m$ lies in $\Rcal_2$ defined in \eqref{eq:R2} and Algorithm \ref{alg:VMC} updates with the schedule \eqref{eq:schedule} and parameters in Table \ref{table:1}, then the expected value of $\Lcal(\tparameter_{m+1})$ taken over the randomness of $\{\parameter_{k}\}_{k=m\cdot T+1}^{(m+1)\cdot T}$ decreases as 
    \begin{equation*}
        \Lcal(\tparameter_{m})-\Ebb[\Lcal(\tparameter_{m+1})|\tparameter_{m}]\geq \Lthre.
    \end{equation*}
\end{lemma}
\begin{proof}
\newcommand{\asumgh}{A_1}
    
    The proof is by contradiction. Without loss of generality, we suppose that $m=0$ in this proof and denote
    \begin{equation*}
        \Lcal_p=\Lcal\left( \parameter_p\right),~~g_{p}=\nabla_{\parameter}\Lcal\left( \parameter_p\right),~~
            \hat{g}_{p}=\hat{g}\left( \parameter_p,S_{p}\right),~~ H_p=\nabla^2_{\parameter}\Lcal\left( \parameter_p\right)
    \end{equation*}
    for $p=0,\dots,T-1$. Every expectation in this proof is taken over all existed $\parameter_{p}$ in every formula. We assume the expected function value decreases by no more than $\Lthre$, i.e.,
    \begin{equation}
        \label{eq:fakedecrease}
        \Lcal_0-\epct{\Lcal_{T }}< \Lthre.
    \end{equation}
    We proceed to show that the assumption \eqref{eq:fakedecrease} is invalid.

    We start with estimating the expected distance between $\parameter_0$ and $\parameter_p$ for $p=1,\dots,T-1$. By Lemma \ref{lem:decrease}, it holds
    \begin{equation}
        \label{eq:R2dec1}
        \begin{aligned}
            &  ~ 2\left(\Lcal_0-\epct{\Lcal_{T }}\right)\\
            &\geq  ~ \beta\norm{g_0}^2-\beta B^2_{n,n_0}-\beta^2 L V_{n,n_0}
            +\sum_{h=1}^{T -1}\left(\alpha \Ebb\norm{g_h}^2-\alpha B^2_{n,n_0}-\alpha^2 L V_{n,n_0}\right)\\
            &\geq ~ \beta\norm{g_0}^2 + \alpha \sum_{h=1}^{T-1}\Ebb\norm{g_h}^2- 2T\alpha B^2_{n,n_0}-2\beta^2LV_{n,n_0},
        \end{aligned}
    \end{equation}
    where the first inequality is derived similarly to \eqref{eq:R1ineq}, and the second inequality is due to $\beta=\sqrt{T}\alpha$.
    Together with the assumption \eqref{eq:fakedecrease}, it follows that
    \begin{equation}
        \label{eq:R2dec2}
        \beta\norm{g_0}^2+\alpha\sum_{h=1}^{T-1}\Ebb\norm{g_h}^2\leq 2\Lthre+2T\alpha B^2_{n,n_0}+2\beta^2 L V_{n,n_0}=:\asumgh.
    \end{equation}
    A direct implication of \eqref{eq:R2dec2} is that
    \begin{equation}
        \label{eq:R2dec3}
        \beta\norm{g_0}^2\leq \asumgh, \qquad
        \alpha\Ebb\norm{\sum_{h=1}^{p}g_h}^2\leq p\alpha\sum_{h=1}^{p}\Ebb\norm{g_h}^2\leq p\asumgh.
    \end{equation}
    
    We proceed to bound $\parameter_{p+1}-\parameter_0=\beta \hat{g}_0+\alpha\sum_{h=1}^{p}\hat{g}_h$ as follows. Firstly,
    \begin{equation}
        \label{eq:distdecomp2}
        \begin{aligned}
            \Ebb\norm{\sum_{h=1}^{p}(\hat{g}_h-g_h)}^2
            &= \sum_{h=1}^{p}\Ebb\norm{\hat{g}_h-g_h}^2+2\sum_{1\leq h<l\leq p}\Ebb\left\langle \hat{g}_h-g_h,\hat{g}_l-g_l\right\rangle\\
            &= \sum_{h=1}^{p}\Ebb\norm{\hat{g}_h-g_h}^2+2\sum_{1\leq h<l\leq p}\Ebb\left\langle \hat{g}_h-g_h,\Ebb[\hat{g}_l|\parameter_l]-g_l\right\rangle\\
            &\leq pV_{n,n_0}+p(p-1)B_{n,n_0}\sqrt{V_{n,n_0}}
            \leq 2pV_{n,n_0}+p^3B_{n,n_0}^2,
        \end{aligned}
    \end{equation}
    where the second equality is because we can take conditional expectation on $\parameter_l$ first, the following inequality holds from the result $\Ebb\norm{\hat{g}_h-g_h}^2\leq V_{n,n_0}$ and $\norm{\Ebb[\hat{g}_l|\parameter_l]-g_l}\leq B_{n,n_0}$ in Theorem \ref{thm:gbiasVar}, and the last inequality is due to AM-GM inequality. Therefore, we can bound
    \begin{equation}
        \label{eq:distbound}
        \begin{aligned}
            &~\Ebb\norm{\parameter_{p+1}-\parameter_0}^2
            =\Ebb\norm{\beta\hat{g}_0+ \alpha\sum_{h=1}^{p}\hat{g}_h}^2\\
            &=\Ebb\norm{\beta g_0+\beta(\hat{g}_0-g_0)+ \alpha\sum_{h=1}^{p}g_h+\alpha\sum_{h=1}^{p}(\hat{g}_h-g_h)}^2\\
            &\leq \underbrace{ 4\beta^2\norm{g_0}^2 }_{\eqref{eq:R2dec3}}
            +\underbrace{ 4\beta^2\Ebb\norm{\hat{g}_0-g_0}^2 }_{\text{Theorem \ref{thm:gbiasVar}}}
            +\underbrace{ 4\alpha^2\Ebb\norm{\sum_{h=1}^{p}g_h}^2 }_{\eqref{eq:R2dec3}}
            +\underbrace{ 4\alpha^2\Ebb\norm{\sum_{h=1}^{p}(\hat{g}_h-g_h)}^2 }_{ \eqref{eq:distdecomp2} }\\
            &\leq 4\beta \asumgh + 4\beta^2 V_{n,n_0} + 4\alpha p \asumgh + 2\alpha^2(2pV_{n,n_0}+p^3B_{n,n_0}^2)\\
            &\leq 4\beta \asumgh + 8\beta^2 V_{n,n_0} + 4 p (\alpha\asumgh+ \alpha^2 T^2 B_{n,n_0}^2 ),
        \end{aligned}
    \end{equation}
    where the second inequality is due to \eqref{eq:R2dec3}, Theorem \ref{thm:gbiasVar} and \eqref{eq:distdecomp2}, and the final follows from $p\alpha^2\leq T\alpha^2 = \beta^2$. Thus, we can take 
    \begin{align*}
        A_2&:=8\alpha\Lthre+16\alpha^2 T^2B_{n,n_0}^2+ 8\alpha\beta^2 LV_{n,n_0},\\
        A_3&:=8\beta\Lthre+8\alpha\beta T B_{n,n_0}^2+ 16\beta^2 V_{n,n_0},
    \end{align*}
    and then
    \begin{align}\label{eq:dtheta-linear}
        \Ebb\norm{\parameter_{p}-\parameter_0}^2 \leq (p-1)A_2+A_3.
    \end{align}
    When we take $p=T-1$, \eqref{eq:distbound} shows that the expected distance between $\parameter_0$ and $\parameter_T$ is bounded by a quadratic function of $T$.

    We further prove the expected distance between $\parameter_0$ and $\parameter_T$ grows at least exponentially for $T$, leading to a contradiction. Since $\parameter_{p}$ stays close to $\parameter_0$, the quadratic Taylor approximation of the function $\Lcal$ at $\parameter_0$ is introduced as
    \begin{equation*}
        Q(\parameter):=\Lcal_0+g_0^{\T}(\parameter-\parameter_0)+\frac{1}{2}(\parameter-\parameter_0)^{\T}H_0(\parameter-\parameter_0).
    \end{equation*}
    We denote $Q_p=Q(\parameter_p)$ and $q_p=\nabla_{\parameter}Q(\parameter_p)=g_0+H_0(\parameter_p-\parameter_0)$ for $p=0,\dots,T-1$. Using the Taylor approximation is firstly proposed in \cite{ge2015escaping}. As $\Lcal$ is twice-differentiable with a $\rho$-Lipschitz Hessian, \cite[Lemma 1.2.4]{nesterov2003introductory} gives that 
    \begin{equation}
        \label{eq:Taylorapprox}
        \norm{\nabla_{\parameter}\Lcal(\parameter)-\nabla_{\parameter}Q(\parameter)}\leq \frac{\rho}{2} \norm{\parameter-\parameter_0}^2.
    \end{equation}
    Thus, $\norm{q_h-g_h}^2\leq \frac\rho2\norm{\parameter_h-\parameter_0}^2$.
    To derive the lower bound, $\parameter_{p+1}-\parameter_0$ is decomposed as
    \begin{equation*}
        \label{eq:decompTaylor}
        \begin{aligned}
            \parameter_{p+1}-\parameter_0&=\parameter_{p}-\parameter_0-\alpha\hat{g}_{p}\\
            &=\parameter_{p}-\parameter_0-\alpha q_p-\alpha(\hat{g}_p-g_p+g_p-q_p)\\
            &=(I-\alpha H_0)(\parameter_{p}-\parameter_0)-\alpha g_0-\alpha(\hat{g}_p-g_p+g_p-q_p).
        \end{aligned}
    \end{equation*}
    Let $-\lambda_0<0$ be the minimum eigenvalue of the Hessian $H_0=H(\parameter_0)$, and let $v$ be the unit eigenvector with respect to $-\lambda_0$ (which is deterministic conditional on $\parameter_0$). Then $(I-\alpha H_0)v=(1+\alpha \lambda_0) v=\kappa v$, and hence
    \begin{align*}
        \iprod{v}{\parameter_{p+1}-\parameter_0}= \kappa \iprod{v}{\parameter_{p+1}-\parameter_0} - \alpha \iprod{v}{g_0} - \alpha \iprod{v}{\hat{g}_p-q_p}.
    \end{align*}
    Recursively expanding this equality out, we finally obtain 
    \begin{align*}
        &\iprod{v}{\parameter_{p+1}-\parameter_0}
        = \kappa^p \iprod{v}{\parameter_{1}-\parameter_0} - \alpha \iprod{v}{g_0}\sum_{h=1}^p \kappa^{p-h} - \alpha \sum_{h=1}^p \kappa^{p-h}\iprod{v}{\hat{g}_h-q_h}\notag\\
        & = \kappa^p \bigg[\beta  \underbrace{ \iprod{v}{-\hat{g}_0} }_{u} - \alpha \underbrace{ \frac{1-\kappa^{-p}}{\kappa-1} \iprod{v}{g_0} }_{d_p} - \alpha \underbrace{ \sum_{h=1}^p \kappa^{-h}\iprod{v}{\hat{g}_h-g_h} }_{\xi_p}
        - \alpha \underbrace{ \sum_{h=1}^p \kappa^{-h}\iprod{v}{g_h-q_h} }_{\delta_p} \bigg]. 
    \end{align*} 
    Therefore,
    \begin{equation}
        \label{eq:R2decp}
        \begin{aligned}
            \Ebb\norm{\parameter_{p+1}-\parameter_0}^2
            \geq& \Ebb[\iprod{v}{\parameter_{p+1}-\parameter_0}^2]
            =\kappa^{2p}\Ebb[\left(\beta u-\alpha d_p-\alpha \xi_p-\alpha \delta_p\right)^2]\\
            \geq& \kappa^{2p} \left(\beta^2 \Ebb [u^2] - 2\alpha\beta \epct{ud_p} - 2\alpha\beta\epct{u\xi_p} - 2\alpha\beta\epct{u\delta_p} \right).
        \end{aligned}
    \end{equation}
    
    % Therefore, the lower bound can be estimated by 
    % \begin{equation}
    %     \label{eq:R2decp}
    %     \Ebb\norm{\parameter_{p+1}-\parameter_0}^2\geq \Ebb\norm{u_p}^2-2\alpha \epct{ \iprod{u_p}{\delta_p}}-2\alpha \epct{ \iprod{u_p}{d_p}}-2\alpha  \epct{ \iprod{u_p}{\xi_p}},
    % \end{equation}
    % where we use the fact that $\norm{a+b}^2\geq \norm{a}^2+2\iprod{a}{b}$. 
    
    % Now we bound each term on the RHS of \eqref{eq:R2decp}. We introduce the notation
    % \begin{equation*}
    %     \kappa :=1+\alpha \lambda_{0}, ~~\lambda_{0}=\abs{ \lambda_{\min}\left(H_0\right)}.
    % \end{equation*}
    % Since $0<\alpha<\frac{1}{L}$, the eigenvalue of $I-\alpha H_0$ lies in $(0,\kappa]$. As $1<\kappa<2$, it holds that %\cfn{easy formula, no need to state here (instead, we can state it inline after inequalities if needed)} 
    % \begin{equation}
    %     \label{eq:summation}
    %     \sum_{h=1}^{p}\kappa^{p-h}\leq \frac{2\kappa^p}{\kappa-1},\quad \sum_{h=1}^{p}\kappa^{p-h}h\leq \frac{2\kappa^p}{(\kappa-1)^2}.
    % \end{equation}
    By the Cauchy-Schwarz inequality, Lemma \ref{lem:cnc} implies that 
    \begin{equation}
        \label{eq:right1}
        \begin{aligned}
            \Ebb [u^2]&= \epct{(v^{\T}\hat{g}_0)^2}\geq \mu\sigma_2^2.
        \end{aligned}
    \end{equation}
    where $\mu=\frac{\eta \gamma}{16n}$.
    Next, because $d_p$ is deterministic, the term $\epct{ud_p}$ can be bounded as
    \begin{equation}
        \label{eq:d_p}
        \begin{aligned}
            \epct{ud_p}
            = ~ -d_p\Ebb [\iprod{v}{\hat{g}_0}]
            =& ~-d_p\iprod{v}{g_0}+d_p\Ebb [\iprod{v}{g_0-\hat{g}_0}]\\
            \leq&~ d_p\Ebb [\iprod{v}{g_0-\hat{g}_0}] \leq \frac{l_gB_{n,n_0}}{\kappa-1},
        \end{aligned}
    \end{equation}
    where the first inequality is due to $-d_p\iprod{v}{g_0}=-\frac{1-\kappa^{-p}}{\kappa-1}\iprod{v}{g_0}^2\leq 0$, and the second inequality uses Theorem \ref{thm:gbiasVar}.

    We next upper bound the term $\epct{u\xi_p}$ as follows.
    \begin{equation}
        \label{eq:xi_p}
        \begin{aligned}
            \epct{u\xi_p}
            =& ~\epct{ u\sum_{h=1}^p \kappa^{-h}\iprod{v}{\hat{g}_h-g_h} }
            = \epct{ u\sum_{h=1}^p \kappa^{-h}\iprod{v}{\Ebb[\hat{g}_h|\parameter_h]-g_h} }\\
            \leq&~ \epct{ |u|\sum_{h=1}^p \kappa^{-h}\norm{\Ebb[\hat{g}_h|\parameter_h]-g_h} }
            \leq \epct{ l_g\sum_{h=1}^p \kappa^{-h}B_{n,n_0} }
            \leq \frac{l_g B_{n,n_0}}{\kappa-1}.
            % \leq& \Ebb \left(\sum_{h=1}^p \kappa^{p-h}\norm{\hat{g}_h-g_h}\right)^2 \\
            % \leq& \left(\Ebb \sum_{l=1}^p \kappa^{p-l}\right) \cdot \left(\sum_{h=1}^p \kappa^{p-h}\norm{\hat{g}_h-g_h}^2\right) \\
            % \leq& \frac{\kappa^p}{\kappa-1} \sum_{h=1}^p \kappa^{p-h} \Ebb\norm{\hat{g}_h-g_h}^2
            % \leq \frac{\kappa^{2p}}{(\kappa-1)^2} V_{n,n_0}
        \end{aligned}
    \end{equation}
    where the second inequality is due to the Cauchy-Schwarz inequality, and in the last inequality we use Theorem \ref{thm:gbiasVar}. 

    Finally, we bound the term $\epct{u\delta_p}$:
    \begin{equation}
        \label{eq:delta_p}
        \begin{aligned}
            \epct{u\delta_p}
            =& \epct{ u\sum_{h=1}^p \kappa^{-h}\iprod{v}{g_h-q_h} }
            \leq \epct{ l_g \sum_{h=1}^p \kappa^{-h}\norm{g_h-q_h} }\\
            \leq& \frac{l_g\rho}{2} \epct{ \sum_{h=1}^p \kappa^{-h}\norm{\parameter_h-\parameter_0}^2 }
            \leq \frac{l_g\rho}{2} \sum_{h=1}^p \kappa^{-h}\left(A_2(h-1)+A_3\right) \\
            \leq& \frac{l_g\rho}{2} \left[ \frac{A_2}{(\kappa-1)^2}+\frac{A_3}{\kappa-1} \right],
        \end{aligned}
    \end{equation}
    where the second inequality is due to \eqref{eq:gradientbound}, the third inequality uses \eqref{eq:dtheta-linear}, and the last inequality is because $\sum_{h=1}^p \kappa^{-h}(h-1)\leq \frac{1}{(\kappa-1)^2}$.

    Substituting the four inequalities \eqref{eq:right1}, \eqref{eq:d_p}, \eqref{eq:xi_p} and \eqref{eq:delta_p} into \eqref{eq:R2decp}, we obtain the lower bound as
    \begin{equation*}
        % \label{eq:total}
        \begin{aligned}
            & ~ \Ebb\norm{\parameter_{p+1}-\parameter_0}^2\geq \kappa^{2p} \left(\beta^2 \Ebb [u^2] - 2\alpha\beta \epct{ud_p} - 2\alpha\beta\epct{u\xi_p} - 2\alpha\beta\epct{u\delta_p} \right)\\
            & \geq \kappa^{2p}
            \left(\beta^2 \mu\sigma_2^2 - 2\alpha\beta \frac{l_g B_{n,n_0}}{\kappa-1}- 2\alpha\beta\frac{l_g B_{n,n_0}}{\kappa-1} - 2\alpha\beta\frac{l_g\rho}{2} \left[ \frac{A_2}{(\kappa-1)^2}+\frac{A_3}{\kappa-1}\right] \right).
        \end{aligned}
    \end{equation*}
    According to our settings in Table \ref{table:1}, these conditions are satisfied:
    \begin{align}
        \Lthre\leq \frac{\mu\sigma_2^2\lambda_0\cdot\min\{\beta\lambda_0,1\}}{192l_g\rho},~ B_{n,n_0}^2\leq\frac{\mu\beta\sigma_2^2\lambda_0\cdot\min\{\beta T\lambda_0,1\}}{384\alpha T^2l_g\rho},~ \beta\leq \frac{\mu\sigma_2^2\lambda_0^2L}{384l_g\rho LV_{n,n_0}}.\label{eq:conditions_1}
    \end{align}
    It follows that 
    \begin{equation}
        \label{eq:A2A3}
        A_2\leq \frac{\mu \alpha \beta \sigma_2^2\lambda_0^2}{8l_g\rho},\quad A_3\leq \frac{\mu \beta \sigma_2^2\lambda_0}{8l_g\rho}.
    \end{equation}
    Therefore, we can conclude that
    \begin{equation}
        \label{eq:total-2}
        \begin{aligned}
           \Ebb\norm{\parameter_{p+1}-\parameter_0}^2&\geq 
           \kappa^{2p} \left(\beta^2 \mu\sigma_2^2-\frac{1}{4}\beta^2 \mu\sigma_2^2 -\frac{1}{4}\beta^2 \mu\sigma_2^2-\frac{1}{4}\beta^2 \mu\sigma_2^2\right)=\frac{1}{4}\beta^2 \kappa^{2p}\mu\sigma_2^2.
            % &~~ -8\alpha^2 \beta  l_g \rho \cdot \frac{\kappa^{2p} }{(\kappa-1)^2}\cdot (\alpha V_{n,n_0} +2\Lthre+(\beta+2T \alpha)B^2_{n,n_0}+2\beta^2LV_{n,n_0}),\\
            % &\overset{\eqref{eq:conditions}}{\geq} \frac{\eta\beta^2\kappa^{2p}\sigma_{2}^{2}}{n}-6\cdot \frac{\eta\beta^2\kappa^{2p}\sigma_{2}^{2}}{8n}=\frac{\eta\beta^2\kappa^{2p}\sigma_{2}^{2}}{4n}.
        \end{aligned}
    \end{equation}
    In other word, \eqref{eq:total-2} shows that the expected distance between $\parameter_0$ and $\parameter_{p+1}$ grows at least exponentially for $p$. Substituting $p=T-1$, it leads to a contradiction if the lower bound of $\Ebb\norm{\parameter_T-\parameter_0}^2$ is greater than the upper bound, i.e.,
    \begin{equation}
        \label{eq:contradiction}
         \frac{1}{4}\beta^2 \kappa^{2T}\mu\sigma_2^2\geq(T-1)A_2+A_3.
        % \frac{\eta\beta^2\kappa^{2T }\sigma_{2}^{2}}{4n}> 4 \alpha T\left(\alpha V_{n,n_0}+ 2\Lthre+(\beta+2T\alpha) B^2_{n,n_0}+2\beta^2LV_{n,n_0}\right)+2\beta^2 l_g^2.
    \end{equation}
    %\cfn{the constraints spell as $\beta\leq c_0 \frac{\underline{V}_{n,n_0}}{LV_{n,n_0}}\lambda^2, B_{n,n_0}<<1, \Lthre\leq \lambda^2\beta\underline{V}_{n,n_0}$, and the $n_0,n$ can simply be taken to be $\tbO{1}$.}
    We choose 
    \begin{equation}
        \label{eq:Tconst}
        T=\frac{c}{\alpha \lambda_0}\log \left(\frac{LV_{n,n_0}n}{\eta \alpha \beta \sigma_2} \right)
    \end{equation}
     in Table \ref{table:1} with taking a sufficiently large $c$ and then \eqref{eq:contradiction} implies a contradiction. This completes the proof.
\end{proof}

Finally, we establish the main theorem in this section. It is shown that MCMC-SGD returns approximate second-order points or $\epsilon$-variance points with high probability. This result may explain the observation of convergence to eigenvalues in solving variational eigenvalue problems, where $\epsilon$-variance points means that an eigenvalue is obtained even if the objective function does not approach the global minimum. When an $\epsilon$-variance point is desired in this kind of problems, we should design better neural network architectures to reduce those meaningless second order stability points. 

\begin{theorem}
    \label{thm:efsp}
    Let Assumptions \ref{asm:wavefun} ,\ref{asm:local} ,\ref{asm:unigap} and \ref{asm:sec} hold. For any $\delta\in (0,1)$, with the stepsizes \eqref{eq:schedule} and parameters in Table \ref{table:1}, Algorithm \ref{alg:VMC} returns an $(\epsilon,\epsilon^{1/4})$ approximate second-order stationary point or an $\epsilon^{1/2}$-variance point with probability at least $1-\delta$ after the following steps

    \begin{equation*}
        O\left(\delta^{-4}\epsilon^{-11/2}\log^2\left(\frac{1}{\epsilon\delta}\right)\right).
    \end{equation*}
    
\end{theorem}
\begin{proof}
    Suppose $\Ecal_m$ is the event
    \begin{equation*}
        \Ecal_m:=\left\{\tparameter_m\in \Rcal_1\cup \Rcal_2\right\},
    \end{equation*}
    and its complement is $\Ecal_m^c=\left\{\tparameter_m\in \Rcal_3 
   \right\}$. Let $\Pcal_m$ denote the  probability of the occurrence of the event $\Ecal_m$. 
    
    When $\Ecal_m$ occurs, by Lemmas \ref{lem:R1} and \ref{lem:R2}, we have 
    \begin{equation}
        \label{eq:Edescent}
        \epct{\Lcal(\tparameter_{m})-\Lcal(\tparameter_{m+1})\Big|\Ecal_m}\geq \Lthre.
    \end{equation} On the other hand, when $\Ecal_m^c$ occurs, it follows from \eqref{eq:R1ineq} that
    \begin{equation}
        \label{eq:Ecdescent}
        2\epct{\Lcal(\tparameter_m)-\Lcal(\tparameter_{m+1})\Big|\Ecal_m^c}\geq -2T\alpha B_{n,n_0}^2-2\beta^2LV_{n,n_0}\geq -\delta \Lthre,
    \end{equation}
    where the first inequality is by discarding positive terms in \eqref{eq:R1ineq} and the second inequality is due to the choice $\Lthre\geq \left(2T\alpha B_{n,n_0}^2+2\beta^2LV_{n,n_0}\right)/\delta$ in Table \ref{table:1}. It means that the function value may increase by no more than $\delta \Lthre/2$. When the expectation is taken overall, \eqref{eq:Edescent} and \eqref{eq:Ecdescent} imply that
    \begin{equation}
        \label{eq:totaldescent}
        \epct{\Lcal(\tparameter_{m})-\Lcal(\tparameter_{m+1})} \geq (1-\Pcal_m)\cdot \left(-\frac{\delta \Lthre}{2}\right)+\Pcal_m\cdot \Lthre.
    \end{equation}
    
    Suppose Algorithm \ref{alg:VMC} runs for $K$ steps starting from $\parameter_0$ and there are $M=K/T $ of $\tparameter_{m}$. Let $\Lcal^*$ be the global minimum of $\Lcal(\parameter)$. Summing \eqref{eq:totaldescent} for $m=1,\dots, M$ yields that
    \begin{equation*}
        \Lcal(\parameter_0)-\Lcal^{*}\geq  -\frac{\delta \Lthre M}{2}+\sum_{m=1}^{M}\Pcal_m\cdot \Lthre\Rightarrow\frac{1}{M}\sum_{m=1}^{M}\Pcal_m \leq \frac{\Lcal(\parameter_0)-\Lcal^{*}}{M\Lthre}+\frac{\delta}{2}\leq \delta,
    \end{equation*}
    where the last inequality holds if $K$ satisfies 
    \begin{equation}
        \label{eq:Kconst}
        K\geq \frac{2[\Lcal(\parameter_0)-\Lcal^{*}]T}{\delta \Lthre}=O\left(\delta^{-4}\epsilon^{-11/2}\log^2\left(\frac{1}{\epsilon\delta}\right)\right).
    \end{equation} Hence, the probability of the event $\Ecal_m^c$ occurs can be bounded by 
    \begin{equation*}
        1-\frac{1}{M}\sum_{m=1}^{M}\Pcal_m\geq 1-\delta.
    \end{equation*}
    This proves the statement in Theorem \ref{thm:efsp}.
\end{proof}

\section{Conclusions}
\label{sec:conclusion}

In this paper, we explore the convergence properties of the SGD algorithm coupled with MCMC sampling.
The upper bound of the bias and variance corresponding to the MCMC is estimated by concentration inequalities without resorting to the conventional bounded assumption.
Then, we show that MCMC-SGD achieves first order stationary convergence based on the error analysis of MCMC. It has $O\left(\frac{\log K}{\sqrt{n K}}\right)$ convergence rate with a sufficiently large sample size $n$ after $K$ iterations. Moreover, we conduct a study on the second-order convergence properties of MCMC-SGD under reasonable assumptions. By verifying the CNC condition under errors, we discuss how MCMC-SGD escapes from saddle points and establish the second order convergence guarantee of $O\left(\epsilon^{-11/2}\log^{2} \left(\frac{1}{\epsilon}\right)\right)$ with high probability. Our result explains the observation of convergence to eigenvalues in solving variational eigenvalue problems, thereby demonstrating the favorable second-order convergence properties of the MCMC-SGD algorithm.

There are some potential directions to be concerned for future research. (1) The relationship between non minimum eigenvalues in variational eigenvalue problems and optimization algorithms remains an area ripe for exploration. (2) Our convergence analysis suggests that improving the efficiency of sampling methods is of vital importance for the better performance of stochastic algorithms. (3) The convergence of the natural gradient method and the KFAC method with MCMC samples is also interesting.

\bibliographystyle{plain}
\bibliography{bib.bib}

\end{document}